\documentclass[journal, twoside, final]{IEEEtran}

\usepackage{amssymb, amsmath, amsfonts}
\usepackage{bm, color, siunitx}
\usepackage{booktabs}
\usepackage{setspace}
\usepackage{url}
\usepackage{graphicx}
\usepackage{longtable, threeparttable, tabularx}
\usepackage{makecell}
\usepackage{multirow}
\usepackage{ragged2e}
\usepackage{fancyhdr}
\usepackage{bbding}
\usepackage{cite}
\usepackage{color}
\usepackage{dsfont}
\usepackage{datetime}
\usepackage{float}
\usepackage[caption=false, font=footnotesize]{subfig}
\newtheorem{lemma}{Lemma}
\newtheorem{corollary}{Corollary}
\newtheorem{proposition}{Proposition}
\newtheorem{assumption}{Assumption}
\newtheorem{remark}{Remark}
\newtheorem{theorem}{Theorem}
\allowdisplaybreaks[4]

\makeatletter

\newcommand{\Rmnum}[1]{\expandafter\@slowromancap\romannumeral #1@}
\makeatother

\usepackage{algorithmic}
\usepackage[ruled]{algorithm2e}

\begin{document}
\title{Energy-Efficient Federated Edge Learning For Small-Scale Datasets in Large IoT Networks}
	\author{Haihui Xie,~\IEEEmembership{Member,~IEEE}, Wenkun Wen,~\IEEEmembership{Member,~IEEE}, Shuwu Chen, Zhaogang Shu,~\IEEEmembership{Senior Member,~IEEE}, \\ and Minghua Xia,~\IEEEmembership{Senior Member,~IEEE}
 	\thanks{Received 28 August 2025; revised 12 February 2026; accepted 8 April 2026. The associate editor coordinating the review of this article and approving it for publication was X. Cao.  This work was supported in part by the Techphant-SYSU joint project under Grant SYSU-76120-20260126-0002, in part by the Fujian Industry-University-Institute Cooperation Project of China under Grants 2024H6007 and 2024H6030, and in part by the Overseas Research and Further Education Program for Young and Middle-aged Backbone Teachers from Universities in Fujian Province of China. \textit{(Corresponding authors: Wenkun Wen; Minghua Xia.)}}
	   \thanks{Haihui Xie, Shuwu Chen, and Zhaogang Shu are with the College of Computer and Information Sciences, Fujian Agriculture and Forestry University, and also with the Engineering Research Center of Smart Sensing and Agricultural Chip Technology, Fujian Province University, Fuzhou 350002, China (e-mail: xiehh@fafu.edu.cn; chenshuwu@fafu.edu.cn; zgshu@fafu.edu.cn).}
   	   \thanks{Wenkun Wen is with the R\&D Department, Techphant Technologies Company Ltd., Guangzhou 510310, China (e-mail: wenwenkun@techphant.net).}
   	   \thanks{Minghua Xia is with the School of Electronics and Information Technology, Sun Yat-sen University, Guangzhou 510006, China 	(e-mail: xiamingh@mail.sysu.edu.cn).}
	   
	\thanks{Digital Object Identifier 10.1109/TWC.2026.3683911}
}

\markboth{IEEE Transactions on Wireless Communications} {Xie \MakeLowercase{\textit{et al.}}: Energy-Efficient Federated Edge Learning For Small-Scale Datasets in Large IoT Networks}

\maketitle

\IEEEpubid{\begin{minipage}{\textwidth} \ \\[12pt] \centering 1536-1276 \copyright\ 2026 IEEE. All rights reserved, including rights for text and data mining, and training of artificial intelligence \\ and similar technologies. Personal use is permitted, but republication/redistribution requires IEEE permission. \\
See \url{https://www.ieee.org/publications/rights/index.html} for more information.\end{minipage}}
	
\maketitle
	
\begin{abstract}       
	Large-scale Internet of Things (IoT) networks enable intelligent services such as smart cities and autonomous driving, but often face resource constraints. Collecting heterogeneous sensory data, especially in small-scale datasets, is challenging, and independent edge nodes can lead to inefficient resource utilization and reduced learning performance. To address these issues, this paper proposes a collaborative optimization framework for energy-efficient federated edge learning with small-scale datasets. We first derive an expected learning loss to quantify the relationship between the number of training samples and learning objectives. A stochastic online learning algorithm is then designed to adapt to data variations, and a resource optimization problem with a convergence bound is formulated. Finally, an online distributed algorithm efficiently solves large-scale optimization problems with high scalability. Extensive simulations and autonomous navigation case studies with collision avoidance demonstrate that the proposed approach significantly improves learning performance and resource efficiency compared to state-of-the-art benchmarks.
    \end{abstract}
	

\IEEEpubidadjcol

\section{Introduction} \label{Section1}
\IEEEPARstart{W}{ith} the rapid growth of the Internet of Things (IoT), applications such as smart cities, smart factories, intelligent transportation, and autonomous driving are generating massive volumes of data. Edge intelligence enables local data processing on IoT devices, improving communication efficiency and responsiveness. Wireless and cellular communications are particularly well-suited for such applications due to their ability to support high data rates and wide coverage. Achieving energy-efficient edge intelligence requires not only selecting the right communication technology but also optimizing resource allocation strategies~\cite{10258360}.  
		
Federated learning (FL) has emerged as a key paradigm for edge intelligence, enabling distributed edge devices to collaboratively train models without requiring centralized data aggregation. Despite extensive research on FL~\cite{10444714, 9605599, 9681911, 9252924, 9252927}, most existing designs prioritize learning performance while overlooking communication constraints. Since reliable communication is indispensable in distributed networks, it is crucial to jointly optimize learning objectives and resource allocation.  
		
\subsection{Related Works and Motivation}	
	To ensure reliable communication in energy-efficient and low-cost IoT networks, recent studies have explored resource allocation under constrained environments~\cite{9273233, 9328513, 10026196}. For example, the study \cite{10026196} proposed a sum-rate maximization (SRM) scheme to exploit wireless gains. However, such throughput-oriented schemes often fail to deliver sufficient learning performance~\cite{9606720}. To overcome these limitations, \emph{task-oriented} resource allocation has been investigated~\cite{9151375, 10120724, 10400772, 10084349, 9606667, 10217150, 9970330, 10716739, 9653664}. Specifically, the works~\cite{9151375, 10120724, 10400772} introduced learning-centric power allocation models to allocate limited energy resources more effectively, while~\cite{10084349, 9606667} developed quality-of-training maximization (QoT-Max) schemes to preserve task information under bandwidth-limited inference. Other works~\cite{10217150, 9970330, 10716739} jointly optimized sensing, communication, and computation to maximize accuracy. Additionally,~\cite{9653664} proposed semantic transmission methods that extract key task-related features from multimodal IoT data to accelerate learning.  
 	
	In addition, \emph{task-aware} communication schemes have been studied to improve resource efficiency in federated edge learning frameworks~\cite{9052677, 9163301, 9252954, 9982368, KouWZ0CNW23}. For instance,~\cite{9052677} proposed a federated edge learning model that minimizes energy costs by selecting optimal local data samples, which is solved via the distributed alternating direction method of multipliers (ADMM). The work in~\cite{9163301} introduced gradient norms as indicators of data importance, enabling joint data selection and communication allocation to improve efficiency. Similarly,~\cite{9252954} designed a sample-level importance metric based on maximum categorical probability, pruning less significant samples to reduce energy consumption. Furthermore, hierarchical FL frameworks~\cite{9982368, KouWZ0CNW23} optimized communication scheduling under budget constraints to improve scalability.  

 \IEEEpubidadjcol
  	
	Another challenge arises from the data acquisition process in distributed IoT networks. Heterogeneous sensory data enhances coverage and generality~\cite{10529137}, but most existing studies are designed under \emph{offline modes}, where models rely on static pre-collected datasets~\cite{9151375, 10120724, 10084349, 9606667, 10217150, 9970330, 9653664, 9052677, 9163301, 9252954, 9982368, KouWZ0CNW23}. Offline frameworks, including reinforcement learning methods~\cite{10816163}, can improve efficiency but often suffer from limited adaptability when datasets lack diversity. To address this, online reinforcement learning approaches~\cite{9738819, 10330753} have been proposed, which enable continuous data collection and model adaptation. Techniques such as prioritized experience replay and stochastic policies improve both sample efficiency and exploration, making online learning more robust to environmental variations.  

	Motivated by these insights, analyzing the theoretical relationship between learning performance and resource utilization is essential for federated edge learning. Existing \emph{task-oriented} schemes~\cite{9982368, KouWZ0CNW23, 9151375, 10120724} require large-scale deterministic datasets, making real-time optimization computationally expensive. Meanwhile, \emph{task-aware} methods~\cite{9052677, 9163301, 9252954} often lead to low resource utilization in distributed networks, where nodes operate independently. Moreover, the interplay between task-oriented and task-aware schemes is rarely studied, leading to redundant computation or degraded performance.  

	These challenges motivate the design of an energy-efficient federated edge learning framework that bridges task-oriented and task-aware paradigms, leverages expected learning loss to capture dataset–performance relationships, and adapts resource allocation to dynamic IoT environments.  

\subsection{Our Approach}	
	To address the limitations above, we propose an \emph{energy-efficient federated edge learning framework} designed for large-scale IoT networks with small-scale training datasets. The core idea is to jointly optimize learning objectives and wireless resource allocation while adapting to dynamic environments.  

	First, we develop a cloud-edge-end collaborative framework to balance resource allocation and learning performance. Unlike conventional task-oriented methods that rely heavily on deterministic datasets, we derive an expected learning loss that characterizes the fundamental relationship between the number of training samples and model performance. This provides a theoretical foundation for capturing learning dynamics under limited datasets.  

	Second, to support dynamic data environments, we design a stochastic online learning algorithm that progressively adjusts the dataset size during training. The algorithm leverages a convergence bound, derived via the majorization–minimization (MM) framework, to ensure that model updates remain stable as new mini-batches are incorporated. This enables efficient learning from small-scale data samples while preserving adaptability to real-time end-device conditions.  

	Third, for scalability in large-scale IoT networks, we propose a distributed optimization algorithm that eliminates redundant variables and fixes an optimization direction. This design significantly improves convergence probability compared with conventional ADMM-based approaches, achieving faster training and higher resource utilization.  

	Finally, to validate the effectiveness of our framework, we deploy the proposed algorithms on high-fidelity simulation platforms, including the Intelligent Robot Simulator (IR-SIM)~\cite{9645287} and Car Learning to Act (CARLA)~\cite{9982368}. Experimental results demonstrate that our approach achieves superior convergence speed and energy efficiency, highlighting the practical benefits of resource-aware federated edge learning in real-world autonomous systems.  

\subsection{Main Contributions}
	The main contributions of this paper are summarized below:  
	\begin{itemize}
 		\item[1)] Expected learning loss formulation: We derive an expected learning loss that rigorously characterizes the relationship between training dataset size and learning performance. This provides a principled approach to linking communication resources with model convergence, particularly in small-scale dataset scenarios.  
    
		\item[2)] Stochastic online learning algorithm: We design an adaptive online learning algorithm that dynamically adjusts the dataset size during training. By deriving a convergence bound through the MM framework, the algorithm ensures stable updates and efficient learning from streaming data.  
    
		\item[3)] Distributed optimization for large-scale IoT: We propose a distributed optimization algorithm that eliminates redundant variables and enforces a fixed optimization direction. Compared with conventional ADMM, this approach improves convergence probability, scalability, and resource utilization in large-scale IoT networks.  
    
		\item[4)] Practical validation: We implement the proposed framework on high-fidelity platforms, including IR-SIM~\cite{9645287} and CARLA~\cite{9982368}. Experimental results and case studies demonstrate faster convergence, reduced energy consumption, and improved performance in autonomous navigation and collision-avoidance tasks.  
	\end{itemize}
	
	The rest of this paper is organized as follows. Section~\ref{S2} introduces the system model and formulates the optimization problem. Section~\ref{S3} analyzes the relationship between edge learning performance and the training dataset size, and develops an adaptive learning model for online data collection. Section~\ref{Section-Centralized} proposes a centralized MM-based algorithm for small- and medium-sized IoT networks. For large-scale networks, Section~\ref{Section-Distributed} presents a distributed algorithm and an accelerated momentum-based variant. Section~\ref{S6} provides simulation results and a case study to evaluate the proposed scheme. Finally, Section~\ref{S7} concludes the paper.
	
{\it Notation}: Scalars, column vectors, and matrices are denoted by italic letters, bold lower-case letters, and bold upper-case letters, respectively. The vectors $\bm{\mathsf{1}}$ and $\bm{\mathsf{0}}$ denote all-ones and all-zeros column vectors, respectively. The superscripts $(\cdot)^T$ and $(\cdot)^H$ represent the transpose and Hermitian transpose, respectively. The $\ell_2$-norm of a vector $\bm{x}$ is denoted by $\|\bm{x}\|_2$. The notation $\mathcal{CN}(\bm{0},\varrho\bm{I})$ denotes a circularly symmetric complex Gaussian distribution with mean $\bm{0}$ and covariance matrix $\varrho\bm{I}$. For a set $\mathcal{X}$, $|\mathcal{X}|$ denotes its cardinality, and $\mathcal{Y}\setminus\mathcal{X}$ denotes the set difference. The operator $\mathbb{E}(\cdot)$ denotes expectation, and $\mathcal{O}(\cdot)$ denotes the order of arithmetic operations. The main symbols used throughout the paper are summarized in Table~\ref{S1-T1}.

\begin{table}[t!]
	\scriptsize
	\setstretch{1.1}
	\centering
	\caption{Summary of main notations.}
	\setlength{\tabcolsep}{2pt}
	\renewcommand{\arraystretch}{1}
	\begin{tabular}{!{\vrule width 1pt} c !{\vrule width 1pt} l !{\vrule width1pt}}
		\Xhline{1.2pt}
		\textbf{Symbol} & \textbf{Definition} \\
		\Xhline{1.0pt}
		$\mathcal{K}, \, \mathcal{K}_{i}$ &  User set; user set at edge node $i$ \\
		$\mathcal{X} (t), \, \mathcal{X}_{i} (t)$ & Collected dataset (all devices); dataset at node $i$ \\
		$\bm{w} (t), \, \bm{w}_{i} (t + 1)$ & Global model; local model at node $i$ \\
		$A (t), \, A_{i} (t)$ & Initial number of collected samples at epoch $t$ (all nodes; node $i$) \\
		$F (\bm{w}), \, F_{i} (\bm{w})$ & Global loss; local loss at node $i$ \\
		$\mathcal{D}, \, \mathcal{D}_{i}$ & The whole data distribution, data distribution at edge node $i$ \\
		$D, \, D_{i}$ & Total dataset size; dataset size at node $i$ \\
		$\alpha, \, \alpha_{i}$ & Data-deficit factor (all nodes; node $i$) \\
		$\bar{F}, \, \bar{F}_{i}$ & Expected learning loss (cloud); expected loss at node $i$ \\
		$\bm{p}(t), \, \bm{p}_{i}(t)$ & Power allocation vector; power vector at node $i$ \\
		$G_{k, \, l}$ & Channel gain from user $l$ when decoding user $k$ \\
		$P, \, B, \, T, \, \sigma^{2}$ & Power budget; bandwidth; transmission time; noise power \\
		$\bm{s}_{j}, \, \xi_{1}, \, \xi_{2}, \, L$ & Random sample; SGD noise constants; smoothness constant \\
		$r_{l}, \, r_{c}, \, r_{g}$ & Learning rate; collision rate; goal rate \\
		$T_{l}, \, N_{c}, \, N_{g}$ & Training time; number of collisions; number of goals \\
		$T_{o}, \, N_{a}, \, N_{u}$ & Optimization time; number of avoidances; failed goals \\	
		\Xhline{1.0pt}
	\end{tabular}
	\label{S1-T1}
\end{table}

\section{System Model and Problem Formulation} \label{S2}
	In this section, we first describe the system model of cloud-edge-end collaboration. Then, we formulate the learning-oriented resource allocation problem.	

\subsection{System Model}
	Fig.~\ref{S2-Fig0} illustrates a federated edge learning system consisting of a cloud server with an aggregated parameter vector $\bm{w} (t)$, $I$ edge servers each equipped with $N$ antennas serving user sets $\mathcal{K} \triangleq \{ \mathcal{K}_{1}, \, \mathcal{K}_{2}, \, \cdots, \, \mathcal{K}_{I} \}$, local learning parameters $\{ \bm{w}_{1} (t), \, \cdots, \, \bm{w}_{I} (t) \}$ and variable training datasets $\mathcal{X} (t) \triangleq \{ \mathcal{X}_{1} (t), \, \mathcal{X}_{2} (t), \, \cdots, \, \mathcal{X}_{I} (t) \}$, respectively. Here, edge node $i$ also has an edge server, user set $\mathcal{K}_{i}$, local parameter $\bm{w}_{i} (t)$, and training dataset $\mathcal{X}_{i} (t)$. In particular, each user set contains various IoT devices or sensors. 
	
	Fig.~\ref{S2-Fig0} also demonstrates the complete data collection pipeline, model training, and information exchange. The cloud server coordinates this federated edge learning framework through global parameter aggregation and dynamic power allocation. Specifically, the cloud server synthesizes model updates from edge nodes and optimizes wireless resource utilization. At the edge node, each server primarily collects heterogeneous training data from associated IoT devices, trains a local model using the collected datasets, and then transmits the model parameters to the cloud server. Each edge node is associated with various IoT devices that generate diverse training data through their interactions with the environment. These IoT devices should transmit sensory data to their designated edge servers, balancing communication overhead and data utility. During the iterative processes, edge servers continuously train local models using newly collected data while the cloud server aggregates these updates into an improved global model. This cyclic process continues until convergence, yielding a robust global model that is well-suited for distributed networks.

	In resource-constrained wireless networks, we perform a dynamic power allocation for efficient communications. To maximize the network utility function of long-term average data rates, by recalling the seminal Shannon formula, the achievable data rate of device $k$ can be expressed as \cite{9151375}
		\begin{equation} \label{S2-EQ4}
			R_{k} = \log_{2} \Bigg( 1 + \dfrac{ G_{k, k} p_{k} }{ \sum_{ \ell \in \mathcal{K} \setminus k } G_{k, \ell} p_{\ell} + \sigma^{2} } \Bigg), \, k \in \mathcal{K}_{i},
		\end{equation}
	where $\sigma^{2}$ denotes the variance of additive white Gaussian noise; $G_{k, \ell}$ represents the composite channel gain from the $\ell^{\rm th}$ device to the edge server when decoding data of the $k^{\rm th}$ device, computed as $G_{k, k} = \rho_{k} \| \bm{h}_{k} \|_{2}^{2}$ if $\ell = k$, and $G_{k, \ell} = { \rho_{\ell} | \bm{h}_{k}^{H} \bm{h}_{ \ell } |^{2} }/{ \| \bm{h}_{k} \|^{2}_{2} }$ if $\ell \neq k$, with $\bm{h}_{k} \in \mathbb{C}^{N \times 1}$ being the complex-valued channel fast-fading vector from the $k^{\rm th}$ device to the edge server and $\rho_{k}$ being the path loss of the $k^{\rm th}$ device. By \eqref{S2-EQ4}, the number of collected data samples at edge node $i$ is computed by
		\begin{equation} \label{S2-EQ5}
			|\mathcal{X}_{i}| = \sum_{k \in \mathcal{K}_{i}} \left\lfloor \dfrac{B T R_{k}}{V} \right\rfloor + A_{i} \approx \sum_{k \in \mathcal{K}_{i}} \dfrac{B T R_{k}}{V} + A_{i},
		\end{equation}
	where $B$ denotes the total bandwidth in \si{Hz}; $T$ represents the transmission time in seconds; $V$ is the bit number of each data sample, and $A_{i}$ is the initial number of historical data for the $i^{\rm th}$ edge node.

\begin{figure}[t!]
	\centering
	\includegraphics[width=0.475\textwidth]{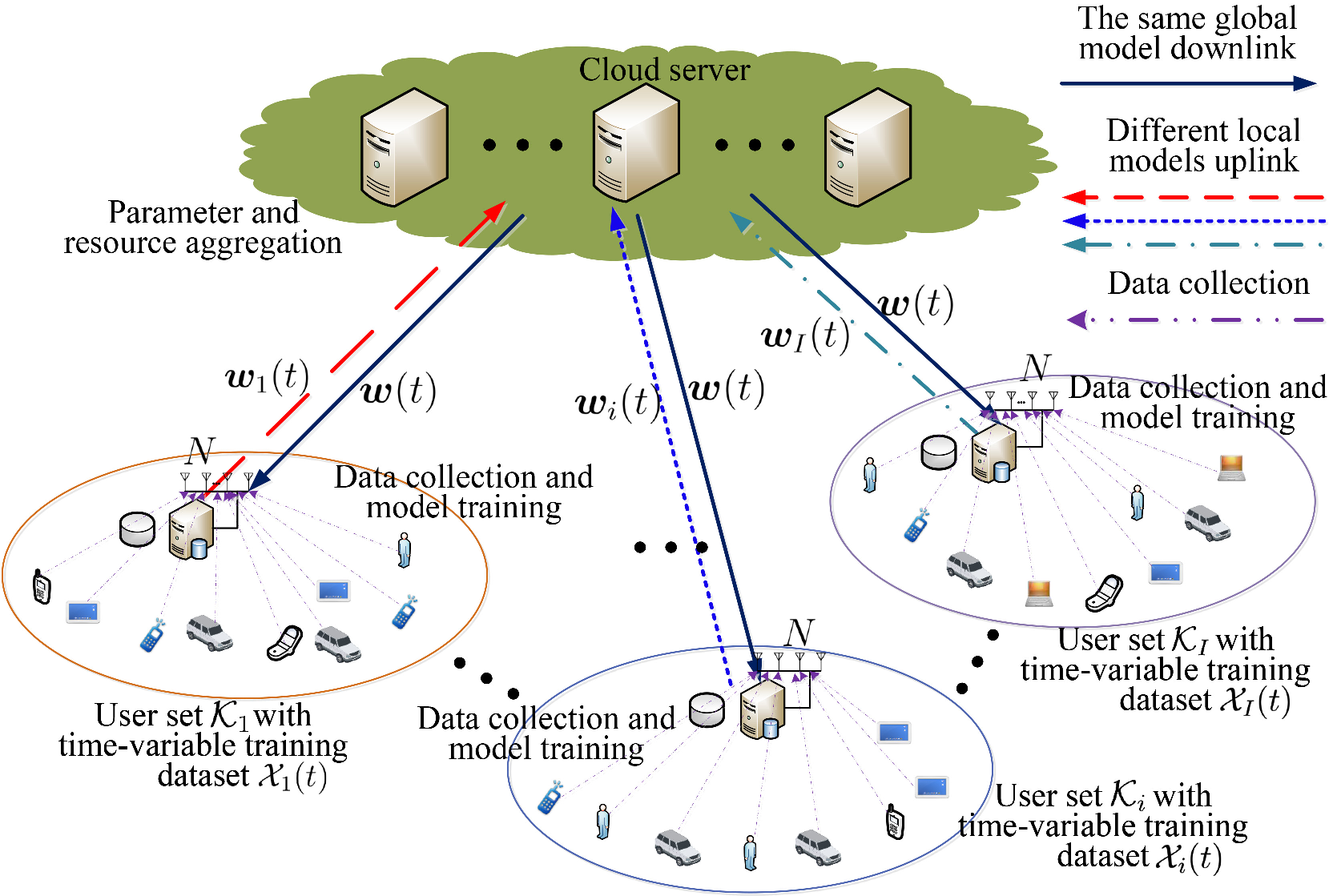}
	\caption{The cloud-edge-end collaborative system in resource-constrained large-scale IoT networks.}
	\label{S2-Fig0}
\end{figure}

\subsection{Problem Formulation}	
	This paper aims to design an energy-efficient federated edge learning model by jointly optimizing the expected loss function and wireless resource allocation strategies. Reinforcement learning typically optimizes a non-parametric learning loss under a parameterized distribution. Motivated by this principle, we introduce a learning loss formulation where the distribution of wireless resource parameters serves as the objective of optimization. To this end, we first derive an expected learning loss that captures the fundamental trade-off between the number of training data samples and the resulting learning performance.

Specifically, consider the federated learning system illustrated in Fig.~\ref{S2-Fig0}, the local loss function at edge server $i$ is defined as
		\begin{equation}
			F_{i} \left( \bm{w} \right) \triangleq \dfrac{1}{\left| \mathcal{D}_{i} \right|} \mathbb{E}_{\mathcal{D}_{i}} \left( \sum_{j \in \mathcal{D}_{i}} f (\bm{w} \mid \bm{s}_{j} ) \right), 
		\end{equation}
	where $\bm{w}$ is a global aggregated parameter; $\bm{s}_{j}$ denotes an arbitrary training data; $f (\cdot \mid \bm{s}_{j})$ is a loss function given the training data $\bm{s}_{j}$, and $\mathcal{D}_{i}$ sufficiently describes the data distribution at edge node $i$. Therefore, $\mathcal{D} \triangleq \cup_{i=1}^{I} \mathcal{D}_{i}$ describes the distribution of the whole dataset \cite{BottouCN18}. Moreover, suppose that $\mathcal{D}_{i} \cap \mathcal{D}_{i^{\prime}} = \emptyset$ for all $i \neq i^{\prime}$, then, the global loss can be expressed as
		\begin{equation}
			F (\bm{w}) \triangleq \frac{1}{D}\sum_{i = 1}^{I} D_{i} F_{i} (\bm{w}) = \sum_{i = 1}^{I} \alpha_{i} F_{i} (\bm{w}),
		\end{equation}
	where $D_{i} \triangleq |\mathcal{D}_{i}|$ and $D \triangleq |\mathcal{D}|$ denote the number of datasets, $\alpha_{i} \triangleq D_{i}/ D$ reflects the fraction of local models that are involved in aggregation.
	
	However, obtaining comprehensive datasets $\mathcal{D}_{i}$ for learning $F_{i} (\bm{w})$, for all $i = 1, \cdots, I$, is a significantly challenge. Therefore, we should employ the collected dataset $\mathcal{X}_{i} (t)$ to asymptotically optimize the true learning model $F_{i} (\bm{w})$. Accordingly, we formulate a resource optimization problem to minimize the learning loss concerning wireless resource constraints, given by 
		\begin{subequations} \label{P1}
			\begin{align}
				\mathcal{P}_{1}: \min_{\bm{p}, \, \{P_{i}\}_{i = 1}^{I}} \ & \dfrac{1}{D} \sum_{i = 1}^{I} D_{i} F_{i} (\bm{w} ( \mathcal{X}_{i} (t) ))  \label{S2-EQ6a} \\
				{\rm s.t.} \
				& |\mathcal{X}_{i} (t)| \leq D_{i}, \ i = 1, \, \cdots, \, I, \label{S2-EQ6b} \\
				& \sum_{i = 1}^{I} |\mathcal{X}_{i} (t)| \leq D, \label{S2-EQ6c} \\
				& \sum_{k \in \mathcal{K}_{i}} p_{k} \leq P_{i}, \, p_{k} \geq 0, \ \forall k \in \mathcal{K}, \label{S2-EQ6d} \\
				& \sum_{i = 1}^{I} P_{i} = P, \label{S2-EQ6e} 
			\end{align}
		\end{subequations}
	where \eqref{S2-EQ6a} implies that the global parameter $\bm{w}$ is determined by $\mathcal{X}_{i} (t)$ due to \eqref{S2-EQ5}; \eqref{S2-EQ6b} and \eqref{S2-EQ6c} denote the maximum local stored capacity of dataset at the edge node; \eqref{S2-EQ6d} denotes that the power allocation of devices adapts to the distributed framework; \eqref{S2-EQ6e} denotes that the power allocation of edge nodes should not exceed the total power budget~$P$ \cite{9151375}. 
	
	Compared with the optimization model with independent resource allocation \cite{9580583}, the proposed model implements the resource optimization more efficiently, since it is deployed at the cloud server, e.g., the gateway of a large-scale IoT network, which informs the devices of their transmit powers and other parameters through the downlink control channel, e.g., the narrow-band physical downlink control channel in NB-IoT networks \cite{zhang2024}. However, the explicit form of $F_{i} (\bm{w})$ related to $\mathcal{X}_{i} (t)$ is undetermined, due to the randomness of $\mathcal{X}_{i} (t)$ and the non-convexity of $F_{i} (\bm{w})$. This implies that finding an exact analytical expression for the power allocation $\bm{p}$ in $\mathcal{P}_{1}$ is mathematically intractable. Thus, we derive an analytical upper bound in Section~\ref{S3-A} to establish the connection between wireless resource allocation and the learning performance.
	
	Note that this work focuses on optimizing the communication energy via transmit power control for online data collection. The energy consumption of local computation (e.g., SGD updates) is not included in $\mathcal{P}_{1}$, since the learning architecture and training hyperparameters are fixed and the computational cost is largely independent of the transmit power decisions. Moreover, our formulation is quantity-oriented, assuming that uploaded samples from each node are representative of its local distribution. We acknowledge that, in heterogeneous IoT environments, computation energy and data importance may both be significant; incorporating these factors is feasible by augmenting the energy model with standard CPU-cycle-based formulations and introducing importance weights estimated from local training feedback (e.g., validation loss or gradient statistics), which is an interesting direction for future work.

\section{Convergence Analysis and Problem Transformation} \label{S3}
	This section first derives the relationship between learning performance and training datasets. Then, we design a learning algorithm that simultaneously collects data and trains a model. Next, we prove the convergence and derive the upper bound of learning algorithms. At last, we transform the original problem $\mathcal{P}_{1}$ into a more mathematical tractable one. 

	\subsection{Relationship Between Learning Performance and Training Datasets} \label{S3-A}	
	We begin by examining the relationship between the performance of the learning model and the training datasets $\mathcal{X}_i(t)$ used to optimize the model parameters $\bm{w}$. The datasets $\mathcal{X}_i(t)$ are collected from the $i^{\rm th}$ edge node in the network and are utilized to approximate the local learning models $F_i(\bm{w})$. However, the optimal performance of these models cannot be rigorously evaluated without knowledge of the underlying data distribution at this optimal state.

To address this, we first analyze the optimal state of the learning model under the conditions provided by $\mathcal{X}_i(t)$, and then establish a connection between the learning performance and the dataset characteristics, as presented in Lemma~\ref{S2-LM1}.
	
	\begin{lemma} \label{S2-LM1}
Let \( X \in \mathcal{D}_i \) be a training dataset and \( Y \in \mathbb{R} \) be the corresponding output (learning loss), where \( \mathcal{D}_i \) denotes the dataset associated with the \( i^{\rm th} \) edge node and \( F_i(\bm{w}) \) is the learning model. Let the joint probability density function of \( X \) and \( Y \) be denoted by \( f_{X, Y}(x, y) \). Then, the expected learning loss \( \bar{F}_i(\bm{w}) \) at the \( i^{\rm th} \) edge node is given by the conditional expectation:
\begin{equation} \label{S2-EQ6}
	\bar{F}_i(\bm{w}) = \mathbb{E}_{f_{X,Y}(x,y)} \left( Y \mid X \in \mathcal{D}_i \right).
\end{equation}	
\end{lemma}

\begin{IEEEproof}
Given the dataset $X$, we define the squared error loss function as:
\begin{equation}
	L(Y, F_i(\bm{w} \mid X)) = (Y - F_i(\bm{w} \mid X))^2,
\end{equation}	
where $F_i(\bm{w} \mid X)$ is the model's prediction based on the dataset $X$.

Using the law of total expectation, we decompose the expected squared error as follows:
\begin{equation}
	\mathbb{E}\left((Y - F_i(\bm{w} \mid X))^2 \right) = \mathbb{E}_X \mathbb{E}_{Y|X}\left((Y - F_i(\bm{w} \mid X))^2 \mid X \right).
\end{equation}	
This expresses the total expected loss as the expectation over the input dataset \( X \) and the conditional expectation over the outputs \( Y \), given \( X \).

The goal is to minimize the expected squared error with respect to the model prediction \( F_i(\bm{w} \mid X) \). This minimization problem is equivalent to finding the value \( \ell \) that minimizes:
\begin{equation}
	\bar{F}_i(\bm{w}) = \underset{\ell}{{\rm argmin}} \, \mathbb{E}_{Y \mid X = x} \left( (Y - \ell)^2 \mid X = x \in \mathcal{D}_i \right),
\end{equation}	
which leads to the optimal prediction being the conditional expectation:
\begin{equation}
	\ell = \mathbb{E}[Y \mid X = x].
\end{equation}	
Thus, the expected learning loss is given by the conditional expectation \( \mathbb{E}[Y \mid X] \), which completes the proof.
\end{IEEEproof}

Lemma~\ref{S2-LM1} shows that the expected learning loss $\bar{F}_{i} (\bm{w})$ is the optimal prediction given by \eqref{S2-EQ6}, while using the collected dataset $\mathcal{X}_{i} (t)$ to optimize $\mathcal{P}_{1}$. Thus, we can use \eqref{S2-EQ6} to define the objective function for the training dataset, especially for small-scale datasets. Although \eqref{S2-EQ6} is an implicit mathematical expression, a fitting nonlinear model can be used to estimate its shape \cite{9151375, 10120724}. However, this fitting model typically relies on large-scale datasets, making it challenging to fit it accurately with the limited data available. To address this issue, Section~\ref{Section-Learning} introduces a framework where data collection and model training occur concurrently. Additionally, we derive an upper bound for \eqref{S2-EQ6} to optimize resource allocation under a learning-oriented principle.

\begin{remark}[Parameter Fitting of Expected Learning Losses] \label{S3-R2}
Consider a sequence of nested sample datasets $\mathcal{X}(t_1) \subseteq \mathcal{X}(t_2) \subseteq \cdots \subseteq \mathcal{X}(t_Q)$ collected over ordered time periods $t_1 \leq t_2 \leq \cdots \leq t_Q$. For each $\mathcal{X}(t_m)$, we train the model to convergence and record the learning loss $\ell_m$. Using the pairs $\{ \mathcal{X}(t_m), \ell_m \}_{m=1}^Q$, the parameters $(a,b)$ of the expected loss model \cite{9151375, 10120724}
\begin{equation} \label{S3-EQ6a}
\bar{F}(\mathcal{X}(t)\mid a,b)\approx a|\mathcal{X}(t)|^{-b}
\end{equation}
are fitted via nonlinear least squares:
\begin{equation} \label{S3-EQ6}
\min_{a>0,b>0}\ \frac{1}{Q}\sum_{m=1}^{Q}\left|\ell_m-\bar{F}\!\left(\mathcal{X}(t_m)\mid a,b\right)\right|^2.
\end{equation}

This fitting can be efficiently solved by grid search or gradient-based methods. The adopted power-law form serves as a surrogate diminishing-return model, which is consistent with classical generalization and optimization scaling laws, where the expected excess risk often decays with sample size as $\mathcal{O}(|\mathcal{D}|^{-1/2})$ or $\mathcal{O}(|\mathcal{D}|^{-1})$ under standard assumptions \cite{BottouCN18, Shalev14, Bubeck15}, and with SGD training where the stochastic gradient variance reduces as more samples become available \cite{BottouCN18}. Motivated by these scaling behaviors, we adopt $\bar{F}(\mathcal{X}(t)\mid a,b)\approx a|\mathcal{X}(t)|^{-b}$, where $(a,b)$ are task-dependent and fitted from data rather than assumed universal. The robustness of the proposed framework to learning-loss model mismatch is empirically validated in Figs.~\ref{S6-Fig3} and \ref{S6-Fig8}, where stable training behavior and consistent performance gains are observed under heterogeneous learning dynamics.

When the dataset turns large, the learning curve saturates, and the online optimization can be simplified (e.g., longer control intervals, fixed/quantized power, or early stopping). In extremely data-scarce regimes, the fitting of $(a, b)$ may be noisy; the proposed algorithm mainly relies on the monotonic diminishing-return structure and continuously refines the fitting online as more samples become available. A short warm-up (data-acquisition) phase or conservative default parameters can also be used to improve robustness in very small-sample regimes.
\end{remark}

\subsection{Learning Algorithm} \label{Section-Learning}
	In contrast to traditional static learning problems, $\mathcal{P}_{1}$, as given by \eqref{P1}, presents a stochastic learning problem with dynamic optimization objectives and datasets. Thus, we design a stochastic online learning algorithm to address this issue. In principle, this algorithm sequentially processes incoming data, making decisions or updates based on stochastic samples or noisy observations. Specifically, we regard $\mathcal{X}_{i} (t) \subseteq \mathcal{D}_{i}$ as a random mini-batch dataset,\footnote{ As the number of iterations increases, the collected sample sets grow progressively over time. For analytical tractability, we treat each sampled dataset as a mini-batch, thereby mitigating stochastic variations during training. In practice, even with a large-scale dataset, mini-batch learning remains an effective approach for efficiently training models while maintaining stability.} then employ the stochastic gradient descent to minimize $F (\bm{w})$:
		\begin{subequations}
			\begin{align}
				\bm{w}_{i} (t + 1) &{}= \bm{w}_{i} (t) - \dfrac{\eta}{|\mathcal{X}_{i} (t)|} \sum_{j \in \mathcal{X}_{i} (t)} f \left(\bm{w} (t) \mid \bm{s}_{j}\right), \label{S2-EQ3a} \\
				\bm{w} (t + 1) &{}= \sum_{i = 1}^{I} \alpha_{i} \bm{w}_{i} (t + 1), \label{S2-EQ3b}
			\end{align} 
		\end{subequations}
	where \eqref{S2-EQ3a} denotes that edge node $i$ updates the learning parameter $\bm{w}_{i} (t + 1)$ using the collected dataset $\mathcal{X}_{i} (t)$ at the $t^{\rm th}$ iteration, and \eqref{S2-EQ3b} represents that the cloud server aggregates learning parameters from $I$ edge nodes.
	
	As the training samples $\mathcal{X}_{i} (t)$ in \eqref{S2-EQ3a} are not dedicated and change over time, this learning process can supplement the dataset of small training samples periodically. To ensure the reliability and stability of the stochastic online gradient descent algorithm, we provide theoretical guarantees for convergence analysis in the following Section~\ref{Section-Convergence}.

\subsection{Convergence Analysis} \label{Section-Convergence}
Now, we analyze the convergence behavior of the learning algorithm as the mini-batch dataset size increases. For this, we make the following standard assumptions.
\begin{assumption} \cite[Assumption 4.1]{BottouCN18} \label{S3-AS1}
The global loss function $F: \mathbb{R}^{d} \rightarrow \mathbb{R}$ is continuously differentiable and has an $L$-Lipschitz continuous gradient, i.e.,
\begin{equation} \nonumber
\| \nabla F ( \bm{x} ) - \nabla F ( \bm{y} ) \|_{2} \leq L \| \bm{x} - \bm{y} \|_{2}, \quad \forall \bm{x},\bm{y}\in\mathbb{R}^d,
\end{equation}
where $L>0$ is the smoothness constant of $F$ (also referred to as the gradient Lipschitz constant), which characterizes the smoothness (curvature upper bound) of the loss landscape induced by the adopted model architecture and regularization.
\end{assumption}

\begin{assumption} \cite[Assumption 4.3]{BottouCN18} \label{S3-AS2}
There exist scalars $\xi_{1} > 0$ and $\xi_{2} > 0$ such that
\begin{equation} \nonumber
\mathbb{E}_{ {\bm s}_{j} } \!\left[ \| \nabla f \left(\bm{x} \mid \bm{s}_{j} \right) \|_{2}^{2} \right]
\leq \xi_{1} + \xi_{2} \| \nabla F ( \bm{x} ) \|_{2}^{2}, \quad \forall \bm{x}\in\mathbb{R}^d,
\end{equation}
where $\nabla f(\bm{x}\mid \bm{s}_{j})$ is the stochastic gradient evaluated on a random sample $\bm{s}_{j}$.
Here, $\xi_1$ quantifies the intrinsic stochastic gradient noise (i.e., a variance-related term that remains even near stationarity), while $\xi_2$ controls how the second moment of the stochastic gradient scales with the squared norm of the full gradient.
\end{assumption}

Assumption~\ref{S3-AS1} ensures the smoothness of $F(\bm{x})$, which is satisfied by many neural network losses, including convolutional neural networks \cite{VirmauxS18, LiHYWZ20}. Assumption~\ref{S3-AS2} captures the bias--variance tradeoff inherent in stochastic gradient descent \cite{10444714}. Both assumptions are widely adopted in convergence analysis, even for non-convex objectives \cite{9252927, 9252924}.

Under these assumptions, we derive an upper bound on the expected optimality gap $\mathbb{E}\!\left[ F(\bm{w}(t)) - F(\bm{w}^{*}) \right]$, as stated below. In particular, building upon Lemma~1, we incorporate the effect of \emph{online data collection} by explicitly linking the dataset size to time, i.e., $|\mathcal{X}|\rightarrow |\mathcal{X}(t)|$. By introducing the scaling factor $\alpha$ to characterize the data-deficit effect in small-scale dataset regimes, we obtain the $\alpha$-dependent convergence bound in Theorem~\ref{S3-TM1}, which directly couples learning progress with the number of collected samples and enables the subsequent learning-driven power allocation.

\begin{theorem} \label{S3-TM1}
Let $\bm{w}^*$ denote the optimal model parameters and set the learning rate to $\eta = 1/L$. Then, the expected optimality gap satisfies
\begin{align}
	&{}\mathbb{E} \left( F ( \bm{w} (t) ) - F ( \bm{w}^{*} ) \right) \nonumber \\
	&{}\leq \mathbb{E} \Bigg( ( 4 \xi_{2} \alpha )^{t} \left( F ( \bm{w} (0) ) - F ( \bm{w}^{*} ) + \dfrac{2 \alpha \xi_{1}}{ ( 1 - 4 \xi_{2} \alpha ) L } \right) \nonumber \\
	&\hspace{3em}  {}+ \dfrac{2 \alpha \xi_{1}}{( 1 - 4 \xi_{2} \alpha) L} \Bigg), \label{S3-EQ7}
\end{align}
where $\alpha \triangleq (D - |\mathcal{X}(t)|)^2 / D^2$ represents the squared fraction of uncollected data, serving as a data-deficit/statistical-efficiency factor, and $\mathcal{X}(t) \triangleq \cup_{i=1}^{I} \mathcal{X}_i(t)$ denotes the total dataset collected from all devices.
\end{theorem}

\begin{IEEEproof}
	See Appendix~\ref{App-Thm1}.
\end{IEEEproof}
	
	Theorem~\ref{S3-TM1} shows that the expected optimality gap at iteration $t$ is upper-bounded by three components: {\it i)} the initial gap $F(\bm{w}(0)) - F(\bm{w}^*)$, independent of the transmitted data size; {\it ii)} a descent term $(4\xi_2 \alpha)^t$, which decreases with more iterations but is only weakly affected by data size; and {\it iii)} a variance term ${2\alpha \xi_1} / ((1 - 4\xi_2 \alpha) L)$, which is strongly dependent on the number of transmitted samples. To tighten the bounds and improve convergence, it is essential to reduce both the descent and variance terms—primarily by increasing the size of the collected dataset $\mathcal{X}(t)$. Intuitively, the data-deficit factor $\alpha$ quantifies how far the currently available training dataset is from the ``statistically sufficient'' regime: a larger $\alpha$ corresponds to a more severe sample deficiency and thus slower effective convergence, whereas $\alpha \to 0$ indicates that the system approaches a data-sufficient regime where classical learning behavior is recovered.

It is noteworthy that Theorem~\ref{S3-TM1} builds on standard SGD convergence tools; the key novelty is that the factor $\alpha$ explicitly captures the data-deficit effect under small-scale datasets, thereby coupling learning progress with the online data-collection process. We also emphasize that the bound is worst-case and mainly used to reveal the monotonic diminishing-return dependence on the number of collected samples, rather than to provide a numerically tight prediction. The constants $\xi_1$ and $\xi_2$ are problem-dependent and vary with the model architecture, optimizer, and data distribution. In practice, they can be estimated during a short warm-up (data-acquisition) phase and further refined online. The proposed framework does not require their exact closed-form values; conservative estimates yield only conservative resource allocation.

To ensure convergence of the training process, we derive the following corollary based on Theorem~\ref{S3-TM1}.

\begin{corollary} \label{S3-C1}
Training convergence is guaranteed if the number of transmitted samples \( \mathcal{X}(t) \) satisfies the lower bound:
\begin{equation} \label{S3-EQ8}
    \left| \mathcal{X}(t) \right| > D - \frac{D}{2\sqrt{\xi_2}}.
\end{equation}
\end{corollary}

\begin{IEEEproof}
From Theorem~\ref{S3-TM1}, convergence is ensured if the descent factor decays to zero as \( t \to \infty \), i.e.,
	\begin{align} \label{S3-EQ10}
		&\lim\limits_{t \rightarrow \infty}\left( 4 \xi_{2} \left( \dfrac{ D - |\mathcal{X} (t)| }{D} \right)^{2} \right)^{t} = 0 \nonumber \\
		& \Longrightarrow \, 4 \xi_{2} \left( \dfrac{ D - |\mathcal{X} (t)| }{D} \right)^{2} \ll 1.
	\end{align}
Solving this inequality yields the condition in \eqref{S3-EQ8}.
\end{IEEEproof}

Corollary~\ref{S3-C1} highlights that a minimum number of initial training samples is necessary to guarantee stable convergence. This emphasizes the importance of initializing each model with sufficiently informative samples \( A_i \) as in \eqref{S2-EQ5}. When a new node joins with very few samples ($A_i \approx 0$), it is placed into a short warm-up (data-acquisition) phase to collect sufficient local data. The edge orchestrator applies a simple admission control rule and only includes nodes with $A_i(t)\ge A_{\min}$ in the online power allocation and training loop. This ensures that the learning updates remain well-conditioned and that the assumptions underlying Theorem~\ref{S3-TM1} are satisfied for the active node set. Additionally, Theorem~\ref{S3-TM1} implies that a tighter upper bound on the convergence gap leads to greater algorithmic stability. Therefore, minimizing this bound becomes a key objective in subsequent Section~\ref{Section-Trans}.

Furthermore, Corollary~\ref{S3-C1} reveals a phase transition in the convergence condition based on the learning loss parameter \( \xi_2 \). When \( 0 < \xi_2 < \tfrac{1}{4} \), the right-hand side of \eqref{S3-EQ8} is negative, indicating a looser convergence requirement and tolerating small initial sample sizes. In contrast, if \( \xi_2 > \tfrac{1}{4} \), the lower bound becomes positive, enforcing a stricter requirement on initial data size for convergence.

To bridge these regimes, we propose an online learning framework that incrementally satisfies the convergence condition by progressively increasing the dataset size. While this approach may slow convergence, it ensures robustness across varying system parameters and model characteristics.

\begin{remark}[Dataset Distribution Shifts]
The analysis in this work assumes stationary local data distributions $\{D_i\}$. In dynamic IoT environments, local distributions may evolve over time, i.e., $D_i \rightarrow D_i(t)$, making the learning objective time-varying. The proposed algorithm remains applicable, as it updates the model using newly collected samples and online stochastic gradients; however, the theoretical guarantee should be interpreted in a tracking/piecewise-stationary sense rather than as convergence to a single fixed optimum. In practice, adaptivity to drift can be improved by periodically re-fitting the surrogate learning-loss model using a sliding window or an exponential forgetting factor, and refreshing the allocation state upon drift detection.
\end{remark}

\subsection{Problem Transformation} \label{Section-Trans}
	The convergence analysis in Theorem~\ref{S3-TM1} highlights the impact of the number of transmitted samples on learning performance. Combining \eqref{S3-EQ7} with \eqref{S2-EQ5}, it is evident that power control also affects learning outcomes. Accordingly, we relax problem \(\mathcal{P}_1\) to the following:
	\begin{align}
		\mathcal{P}_{2}: \min_{\bm{p}} \ & \left( 4 \xi_{2} \alpha \right)^{t} \left( C + \dfrac{2 \alpha \xi_{1}}{ ( 1 - 4 \xi_{2} \alpha \left( |\mathcal{X} (t)|, \bm{p} \right) ) L } \right)  \nonumber \\
			& \quad {}+ \dfrac{2 \alpha \xi_{1}}{ \left( 1 - 4 \xi_{2} \alpha \right) L}  \\
			{\rm s.t.} \ &\eqref{S2-EQ6b}\text{-}\eqref{S2-EQ6d}, \, \eqref{S2-EQ5}, \nonumber
	\end{align}
	where \( C \triangleq F(\bm{w}(0)) - F(\bm{w}^*) \) denotes the initial loss gap.
	
	Under the conditions of Theorem~\ref{S3-TM1} and Corollary~\ref{S3-C1} (i.e., \eqref{S3-EQ8}), \(\mathcal{P}_2\) can be further simplified to	 
	 {\small \begin{align}
		\mathcal{P}_{3}: \min_{\bm{p}} \ & \underbrace{ \left( \dfrac{B T}{V} \sum_{k \in \mathcal{K}} \log_{2} \Bigg( 1 + \dfrac{ G_{k, k} p_{k} }{ \sum_{ \ell \in \mathcal{K} \setminus k } G_{k, \ell} p_{\ell} + \sigma^{2} } \Bigg) + A - D \right)^{2}}_{\triangleq \, \Phi (\bm{p})}  \\
		{\rm s.t.} \ & \eqref{S2-EQ6b} \text{-} \eqref{S2-EQ6d}, \eqref{S2-EQ5} \nonumber
	\end{align}}
\hspace{-1em} where \( A \triangleq \sum_{i=1}^I A_i \) is the total number of initial samples. We assume \( A_i \geq D_i - {D_i}/{(2 \sqrt{\xi_2})} \), which suffices to ensure stable convergence according to Corollary~\ref{S3-C1}.

\section{Centralized MM-based Algorithm} \label{Section-Centralized}
	The problem $\mathcal{P}_{3}$ remains a nonlinear integer programming problem that is challenging to solve due to the presence of an exponential term. To address this issue, we propose an iterative MM-based algorithm that sequentially solves optimization sub-problems until convergence. Specifically, we construct a sequence of upper bounds $\{\widetilde{\Phi}\}$ on $\{\Phi\}$ and substitutes $\{\Phi\}$ in $\mathcal{P}_{3}$ with $\{\widetilde{\Phi}\}$ to obtain the surrogate problems. Given any feasible solution $\bm{p}^{*}$ to $\mathcal{P}_{3}$, we define surrogate functions $\widetilde{\Phi} (\bm{p} \mid \bm{p}^{*})$ as Eq.~\eqref{S4-EQ12} shown at the top of the next page, and the following proposition can be established.
	
 \begin{figure*}[!th]
		\begin{align}
			\widetilde{\Phi} (\bm{p} \mid \bm{p}^{*}) \triangleq {}&{} \left( \dfrac{BT}{V \ln 2} \sum_{k = 1}^{K} \left( \ln \left( \sum_{\ell = 1}^{K} \dfrac{G_{k, \ell} p_{\ell}}{\sigma^{2}} + 1 \right) - \ln \left( \sum_{\ell = 1, \, \ell \neq k}^{K} \dfrac{G_{k, \, \ell} p_{\ell}^{*}}{\sigma^{2}} + 1 \right)  + 1 \right. \right. \nonumber \\
			{}&{} - \left. \left. \left( \sum_{\ell = 1, \, \ell \neq k}^{K} \dfrac{G_{k, \, \ell} p_{\ell}^{*}}{\sigma^{2}} + 1 \right)^{-1} \times \left( \sum_{\ell = 1, \, \ell \neq k}^{K} \dfrac{G_{k, \, \ell} p_{\ell}}{\sigma^{2}} + 1 \right) \right) + A - D \right)^{2}, \label{S4-EQ12}
		\end{align}
\hrulefill
\end{figure*}	

\begin{figure*}[!t]
 	\vspace{-10pt}
	\begin{subequations} 
		\begin{align}
			\mathcal{P}_{4}: \min_{\bm{p}} \ & \widetilde{\Phi} (\bm{p} \mid \bm{p} (t)) \\ 
			{\rm s.t.} \ &\widetilde{\Phi} (\bm{p} \mid \bm{p} (t)) \leq \dfrac{1}{2} D^{2}, \\	
			& \dfrac{BT}{V \ln 2} \sum_{k \in \mathcal{K}_{i}} \left( \ln \left( \sum_{\ell = 1}^{K} \dfrac{G_{k, \ell} p_{\ell}}{\sigma^{2}} + 1 \right) - \ln \left( \sum_{\ell = 1, \, \ell \neq k}^{K} \dfrac{G_{k, \, \ell} p_{\ell} (t)}{\sigma^{2}} + 1 \right) + 1 \right. \nonumber \\
			{}&{} \left. - \left( \sum_{\ell = 1, \, \ell \neq k}^{K} \dfrac{G_{k, \, \ell} p_{\ell} (t)}{\sigma^{2}} + 1 \right)^{-1} \times \left( \sum_{\ell = 1, \, \ell \neq k}^{K} \dfrac{G_{k, \, \ell} p_{\ell}}{\sigma^{2}} + 1 \right) \right) + A_{i} - D_{i} \leq 0, \\
			&\eqref{S2-EQ6d}.
		\end{align}
	\vspace{-10pt}
	\end{subequations}
\hrulefill
\end{figure*}

\begin{proposition}[MM Surrogate: Majorization and Tangency] \label{S4-PR1} For any feasible expansion point $\bm{p}^{*}$, the surrogate function $\widetilde{\Phi}(\bm{p}\mid \bm{p}^{*})$ defined in \eqref{S4-EQ12} satisfies the standard MM conditions:
\begin{itemize}
	\item[(1)] Majorization: $\Phi(\bm{p}) \le \widetilde{\Phi}(\bm{p}\mid \bm{p}^{*})$ for all feasible $\bm{p}$;
	\item[(2)] Tangency: $\widetilde{\Phi}(\bm{p}^{*}\mid \bm{p}^{*}) = \Phi(\bm{p}^{*})$;
	\item[(3)] First-order consistency: $\nabla_{\bm{p}} \widetilde{\Phi}(\bm{p}\mid \bm{p}^{*})\big|_{\bm{p}=\bm{p}^{*}} = \nabla_{\bm{p}} \Phi(\bm{p}^{*})$.
\end{itemize}
\end{proposition}

\begin{IEEEproof}
	See Appendix~\ref{App-Prop1}.
\end{IEEEproof}

\begin{figure}[!t]
	\centering
	\includegraphics[width=0.475\textwidth]{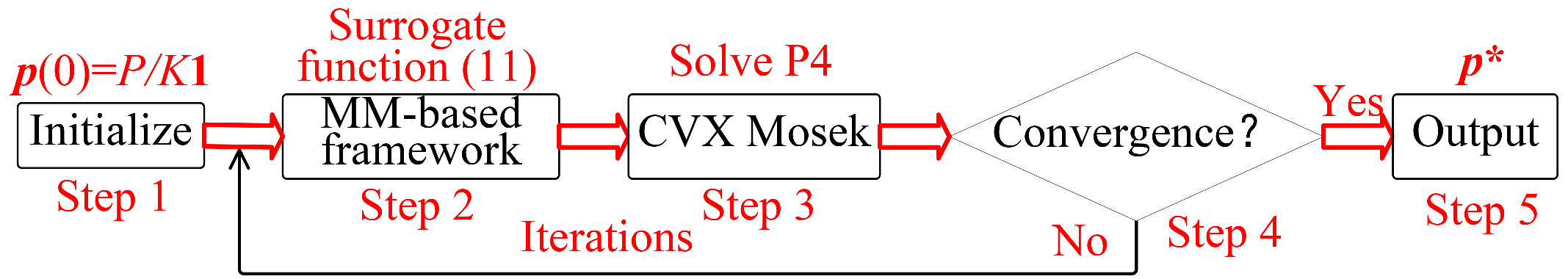}
	\caption{Block diagram of the centralized MM-based algorithm.}
	\label{S4-Fig2}
\end{figure}
								
	With part~$(1)$ of Proposition~\ref{S4-PR1}, an upper bound can be directly obtained by replacing the functions $\{\Phi\}$ by $\{\widetilde{\Phi}\}$ around a feasible point. However, a tighter upper bound can be achieved if we treat the obtained solution as another feasible point, which constructs the next-round surrogate function. Specifically, assuming that the solution at the $t^{\rm th}$ iteration is denoted as $\bm{p} (t)$, the problem $\mathcal{P}_{4}$ is formulated (shown at the top of the next page) and considered at the $(t + 1)^{\rm th}$ iteration. According to part~$(2)$ of Proposition~\ref{S4-PR1}, $\mathcal{P}_{4}$ is convex and can be efficiently solved using an iterative procedure. Part~$(3)$ of Proposition~\ref{S4-PR1} indicates that every limit point of $\{ \bm{p} (t) \}$ is a Karush-Kuhn-Tucker point of $\mathcal{P}_{3}$. 
								
	Fig.~\ref{S4-Fig2} describes the detailed workflow of the centralized MM-based algorithm. Specifically, Step~$1$ initializes the power allocation as $\bm{p} (0) = P / K \bm{1}$, where the symbol $\bm{1}$ represents a column vector with all elements being one. In Step~$2$, the algorithm iteratively constructs a surrogate function using the MM-based framework, as per \eqref{S4-EQ12}. Then, Step~$3$ utilizes standard software packages such as CVX and Mosek to solve $\mathcal{P}_{4}$, and Step~$4$ evaluates whether the solution has converged. Last, Step~$5$ outputs the final solution $\bm{p}^{*}$ while the convergence condition is satisfied.

	Regarding computational complexity, this algorithm exhibits medium or high time complexity, making it suitable for small- or medium-scale IoT networks. Specifically, $\mathcal{P}_{4}$ involves $K$ primal variables and $I + K + 2$ dual variables. The dual variables correspond to $I + K + 2$ constraints in $\mathcal{P}_{4}$, where $I + 1$ constraints come from the sample bounds, $K$ constraints come from non-negative power constraints, and one constraint is derived from the power budget. Consequently, the worst-case complexity for solving $\mathcal{P}_{4}$ is $\mathcal{O} \left( (I + 2K + 2)^{7/2} \right)$, where the number of iterations required for convergence is ignored.
	
	In real-world applications, the proposed algorithm can be deployed at the cloud server, such as the gateway of a small-scale edge network. The server can then inform devices of their transmit powers and other parameters through the downlink control channel, such as the narrowband physical downlink control channel in NB-IoT networks \cite{zhang2024}. This algorithm reduces communication delays during parameter transmission. However, as large-scale IoT networks require optimization for distributed parallelism, communication efficiency across multiple edge nodes tends to degrade. Moreover, the per-iteration complexity also remains high. Thus, in the next Section~\ref{Section-Distributed}, we design a distributed first-order method (FoM) to address these issues.

 \section{FoM-based Distributed Algorithm for Large-scale IoT networks}  \label{Section-Distributed}
    So far, we have addressed $\mathcal{P}_{3}$ using the centralized MM-based algorithm. However, the computational complexity of the solution increases as the number of IoT devices and edge nodes grows. Additionally, the proposed algorithm cannot resolve the coupling among different edge nodes, making distributed deployment challenging. To address these issues, we design a FoM-based distributed algorithm. Due to the dependence on power and co-channel interference terms in large-scale IoT networks, parallelization across edge nodes is difficult.  Leveraging Theorem~\ref{S3-TM1} and Proposition~\ref{S4-PR1}, we further relax $\mathcal{P}_{4}$ to handle these independently. This method yields various sub-problems through variable relaxation, facilitating an easier parallel solution \cite{7889039}.
    
    Now we begin extracting the relevant sub-problems. Given Theorem~\ref{S3-TM1} and Proposition~\ref{S4-PR1}, $\mathcal{P}_{4}$ can be readily relaxed as follows:
\begin{subequations} 
	\begin{align}
		\mathcal{P}_{5}: \min_{ \{ \bm{p}_{i}, \, P_{i} \}_{i =1}^{I} } \ & \dfrac{1}{2} \sum_{i = 1}^{I} \phi_{i} (\bm{p}_{i})^{2} \label{S5-EQ13a} \\
		{\rm s.t.} \ & \phi_{i} (\bm{p}_{i}) \leq 0, \, i = 1, \, \cdots, \, I, \label{S5-EQ13b} \\
			& \bm{1}^{T} \bm{p}_{i} \leq P_{i}, p_{k} \geq 0, \, k \in \mathcal{K}_{i}, \label{S5-EQ13c} \\
			& P_{i} \in \left\{ x_{i} \geq 0 \left| \sum_{i = 1}^{I} x_{i} = P \right\}, \, i = 1, \, \cdots, \, I, \right. \label{S5-EQ13d}
	\end{align}
	 where $\phi_{i} (\bm{p}_{i})$ is defined in \eqref{S5-EQ13e} at the top of the next page.
	 
\begin{figure*}[!t]
	\begin{align}
		\phi_{i} (\bm{p}_{i}) \triangleq {}
			&  \dfrac{BT}{V \ln 2} \sum_{k \in \mathcal{K}_{i}} \left( {} \ln \left( \sum_{\ell \in \mathcal{K}_{i}} \dfrac{G_{k, \ell} p_{\ell}}{\sigma^{2}} + \sum_{\ell \in \mathcal{K} \setminus \mathcal{K}_{i}} \dfrac{G_{k, \ell} p_{\ell} (t)}{\sigma^{2}} + 1 \right) + 1 - \ln \left( \sum_{\ell \in \mathcal{K}, \, \ell \neq k} \dfrac{G_{k, \, \ell} p_{\ell} (t)}{\sigma^{2}} + 1 \right) \right. \nonumber \\
			 &  \left. {}- \left( \sum_{\ell \in \mathcal{K}, \, \ell \neq k} \dfrac{G_{k, \, \ell} p_{\ell} (t)}{\sigma^{2}} + 1 \right)^{-1} \times \left( \sum_{\ell \in \mathcal{K}_{i}, \, \ell \neq k} \dfrac{G_{k, \, \ell} p_{\ell}}{\sigma^{2}} + \sum_{\ell \in \mathcal{K} \setminus \mathcal{K}_{i}, \, \ell \neq k} \dfrac{G_{k, \, \ell} p_{\ell} (t)}{\sigma^{2}} + 1 \right) {} \right) + A_{i} - D_{i}. \label{S5-EQ13e}
	\end{align}        
	\hrulefill
\end{figure*}
\end{subequations}

	Although $\mathcal{P}_{5}$ is a distributed model that can be deployed across different edge nodes, it is non-convex and challenging to solve due to its dependency on the total power across different edge nodes. Thus, we propose a first-order algorithm to solve it. Accordingly, we begin to write the augmented Lagrangian function (ALF) of $\mathcal{P}_{5}$ as follows: 
		\begin{align}
			& L \left( \{ \bm{p}_{i} \}_{i = 1}^{I}, \{ P_{i} \}_{i =1}^{I}; \{\lambda_{i}\}_{i = 1}^{I}, \{\beta_{i}\}_{i = 1}^{I} \right) \nonumber \\
			&= \dfrac{1}{2} \sum_{i = 1}^{I} \phi_{i} (\bm{p}_{i})^{2} + \sum_{i = 1}^{I} \lambda_{i} \phi_{i} (\bm{p}_{i}) \nonumber \\
			&\quad {}+ \sum_{i = 1}^{I} \beta_{i} \left( \bm{1}^{T} \bm{p}_{i} - P_{i} \right) + \dfrac{\mu}{2} \sum_{i = 1}^{I} \left( \bm{1}^{T} \bm{p}_{i} - P_{i} \right)^{2}, \label{S4-EQ14}
		\end{align}
	where $\mu$ is a penalty factor. As shown shortly, \eqref{S4-EQ14} enables the decomposition of sub-problems across different edge nodes. Based on \eqref{S4-EQ14}, we formalize the following proposition.
		\begin{proposition} \label{S5-PR3}
			For the ALF given by \eqref{S4-EQ14}, we can make the following iteration concerning variables $P_{i}$:
				\begin{equation} \label{S5-EQ15}
					P_{i} (t) \triangleq \bm{1}^{T} \bm{p}_{i} (t) - \dfrac{1}{I} \left( \bm{1}^{T} \bm{p} (t) - P \right).
				\end{equation}
			The relative dual variables $\{\beta_{i}\}_{i = 1}^{I}$ are updated by
				\begin{equation} \label{S5-EQ16}
					\beta (t + 1) = \max \left(\beta (t) + \dfrac{\mu}{I} \left( \bm{1}^{T} \bm{p} (t + 1) - P \right), \, 0 \right),
				\end{equation}
			and $\beta_{i} (t + 1) = \beta (t + 1)$, for all $i = 1, \, \cdots, \, I$.
		\end{proposition}
		\begin{IEEEproof}
			See Appendix~\ref{App-Prop2}.
		\end{IEEEproof}
	
	By Proposition~\ref{S5-PR3}, it is evident that we have efficiently dealt with \eqref{S5-EQ13d} and reduced the dimension of dual variables. Next, we decompose the ALF in \eqref{S4-EQ14} into a set of sub-functions, each representing the ALF for the $i^{\rm th}$ edge node, given by
		\begin{align} \label{S5-EQ17}
			L_{i} (\bm{p}_{i}; \lambda_{i}, \beta) 
			& = \dfrac{1}{2} \phi_{i} (\bm{p}_{i})^{2} + \lambda_{i} \phi_{i} (\bm{p}_{i}) \nonumber \\
			&\quad {}+  \beta \left( \bm{1}^{T} \bm{p}_{i} - P_{i} \right) + \dfrac{\mu}{2} \left( \bm{1}^{T} \bm{p}_{i} - P_{i} \right)^{2}.
		\end{align}	
	By \eqref{S5-EQ17}, we compute partial derivatives of the relevant variables and apply the gradient descent algorithms in distributed edge nodes, as detailed below.

\subsubsection{\underline{Update $\bm{p}_{i}$ and $P_{i}$ With Other Variables Fixed}} It is observed that $L_{i} (\bm{p}_{i}; \lambda_{i}, \beta)$ given by \eqref{S5-EQ17} is differentiable with respect to $\bm{p}_{i}$. Then, its gradient can be easily computed as 
		\begin{align}
			\nabla_{\bm{p}_{i}} L_{i} (\bm{p}_{i}; \lambda_{i}, \beta) 
			& = \phi_{i} (\bm{p}_{i}) \nabla_{\bm{p}_{i}} \phi_{i} (\bm{p}_{i}) + \lambda_{i} \nabla_{\bm{p}_{i}} \phi_{i} (\bm{p}_{i}) \nonumber \\
			& \quad {}+ \beta \bm{1} + \mu \left( \bm{1}^{T} \bm{p}_{i} - P_{i} \right) \bm{1}.
		\end{align}
	We then apply the gradient descent algorithm to update $\bm{p}_{i} (t + 1)$, yielding
		\begin{equation} \label{S5-EQ18}
			\bm{p}_{i} (t + 1) = \max (\bm{p}_{i} (t) - \eta \nabla_{\bm{p}_{i}} L_{i} (\bm{p}_{i} (t); \lambda_{i} (t), \beta (t)), \, \bm{0}), 
		\end{equation}
	where $\eta$ denotes the step-size. Also, by Proposition~\ref{S5-PR3}, $P_{i}$ can be efficiently computed using \eqref{S5-EQ15} and thus, $\bm{1}^{T} \bm{p}_{i} (t) - P_{i} (t)$ can be readily updated.

	\subsubsection{\underline{Update Relative Dual Variables With Others Fixed}} It is evident that the ALF given by \eqref{S5-EQ17} is linear with respect to all dual variables. Thus, we update the dual variables as follows:
		\begin{equation} \label{S5-EQ19}
			\lambda_{i} (t + 1) = \max (\lambda_{i} (t) + \eta \phi_{i} (\bm{p}_{i} (t + 1)), \, 0).
		\end{equation}
	Apart from these dual variables, $\beta (t+1)$ can be directly updated using \eqref{S5-EQ16}. In real-world applications, edge node $i$ updates \eqref{S5-EQ18} and \eqref{S5-EQ19} in parallel, and then transmits $\bm{p}_{i} (t + 1)$ and $\lambda_{i} (t + 1)$ to the cloud server. This cloud server updates \eqref{S5-EQ16} and broadcasts $\beta (t + 1)$ to all edge nodes. Compared with the dimensionality of learning parameters, the optimization parameters have a lower dimensionality, making their communication delay negligible.

    So far, we have developed a distributed algorithm to solve $\mathcal{P}_{5}$. However, the convergence rate slows down as multiple variables must be relaxed for parallelism. Therefore, we add momentum to accelerate this first-order algorithm. Specifically, $\bm{p}_{i} (t + 1)$ is updated by
        \begin{subequations}
            \begin{align}
                \bm{z}_{p_{i}} (t + 1) &= \max (\bm{p}_{i} (t) - \eta \nabla_{\bm{p}_{i}} L_{i} (\bm{p}_{i} (t); \lambda_{i} (t), \beta (t)), \, \bm{0}), \label{S5-EQ21a}\\
                \bm{p}_{i} (t + 1) &= (1 - \theta (t)) \bm{p}_{i} (t) + \theta (t) \bm{z}_{p_{i}} (t + 1), \label{S5-EQ21b}\\
                \theta (t + 1) &= \dfrac{1}{2} \left(-\theta^{2} (t) + \sqrt{\theta^{4} (t) + 4 \theta^{2} (t)}\right). \label{S5-EQ21c}
            \end{align}
        \end{subequations}

The procedure above is formalized in {\bf Algorithm~\ref{S5-AM1}}, which is implemented in a distributed manner across edge nodes. In Algorithm~\ref{S5-AM1}, the primal variables $\{\bm{p}_i\}$ are updated locally in parallel, while the auxiliary/dual variables (e.g., $\lambda_i$) and the acceleration-related parameters are updated iteratively based on the current primal iterates.

Regarding computational complexity, the number of decision variables scales linearly with the number of devices $K$. Since most updates are separable across edge nodes, the per-iteration computational cost is $\mathcal{O}(K/I)$ at each edge node, and $\mathcal{O}(K)$ in total across the network, up to constant factors. Moreover, when the surrogate subproblem is convex and smooth, the accelerated first-order updates achieve the standard convergence-rate improvement from $\mathcal{O}(1/k)$ to $\mathcal{O}(1/k^{2})$ with respect to the number of FoM inner iterations $k$.

Moreover, the convergence behavior of Algorithm~\ref{S5-AM1} can be characterized by the following proposition, which summarizes standard descent and rate results for MM with inexact first-order inner solvers.
\begin{proposition}[MM Descent and FoM Convergence {\cite{HunterL04, Nesterov04, BeckT09}}] \label{Prop:FoM}
Let $\{\bm{x}^{(t)}\}$ be generated by the MM framework with surrogate $g(\bm{x}\mid \bm{x}^{(t)})$ satisfying Proposition~\ref{S4-PR1}, and let $\bm{x}^{(t+1)}$ be obtained by approximately minimizing $g(\cdot\mid \bm{x}^{(t)})$ using a FoM until a prescribed accuracy $\epsilon$ is reached (e.g., as in Algorithm~\ref{S5-AM1}). If the inner solution satisfies the sufficient-decrease condition
\begin{equation}
g(\bm{x}^{(t+1)}\mid \bm{x}^{(t)}) \le g(\bm{x}^{(t)}\mid \bm{x}^{(t)}),
\end{equation}
then the MM outer objective is non-increasing:
\begin{equation}
f(\bm{x}^{(t+1)}) \le f(\bm{x}^{(t)}), \quad \forall t.
\end{equation}
Moreover, under standard regularity conditions for MM, every limit point of $\{\bm{x}^{(t)}\}$ is a stationary point of the original problem.

For each fixed outer iterate $t$, if the surrogate subproblem is convex and $L_g$-smooth, then the FoM inner iterates $\bm{x}^{(t,k)}$ satisfy
\begin{equation}
g(\bm{x}^{(t,k)}\mid \bm{x}^{(t)}) - g^\star(\bm{x}^{(t)}) = \mathcal{O}(1/k),
\end{equation}
and the accelerated variant achieves $\mathcal{O}(1/k^2)$ under the usual convexity and smoothness assumptions.
\end{proposition}

\begin{algorithm}[!t]
	\small
	\caption{The FoM-based Distributed Algorithm with an Acceleration.}
	\label{S5-AM1}
	\renewcommand{\algorithmicrequire}{\textbf{Input:}}
	\renewcommand{\algorithmicensure}{\textbf{Output:}}
	\begin{algorithmic}[1]
	\REQUIRE Setting $( I, N, K, P, T, B, V, \sigma^{2}, D_{i}, A_{i} )$, channels $\{ \bm{h}_{k} \}$, user sets $\{ \mathcal{K}_{i} \}$, learning rate $\eta$, optimization parameter $\mu$, and error tolerance $\epsilon$.
	\ENSURE The optimal solution $\bm{p}^{*}$.
	\STATE Initialize $t = 0, \, \bm{p}_{i} (t) = P / I \bm{1}, \, \lambda_{i} (t) = 1, \, \beta (t) = 1$;
	\REPEAT
		\STATE Update $\bm{z}_{p_{i}} (t + 1)$ as per \eqref{S5-EQ21a};
		\STATE Update $\bm{p}_{i} (t + 1)$ as per \eqref{S5-EQ21b};
		\STATE Update $\lambda_{i} (t + 1)$ as per \eqref{S5-EQ19};
		\STATE Update $\beta (t + 1)$ as per \eqref{S5-EQ16};
		\STATE Update $\theta (t + 1)$ as per \eqref{S5-EQ21c};
		\STATE $t = t + 1$;
	\UNTIL{$\| \bm{p}_{i} (t) - \bm{p}_{i} (t - 1) \|_{2} \leq \epsilon$}. 
	\end{algorithmic}
\end{algorithm}
  
\begin{figure*}[t!]
	\centering
	\subfloat[Weather classification.]{\includegraphics[width=0.32\linewidth]{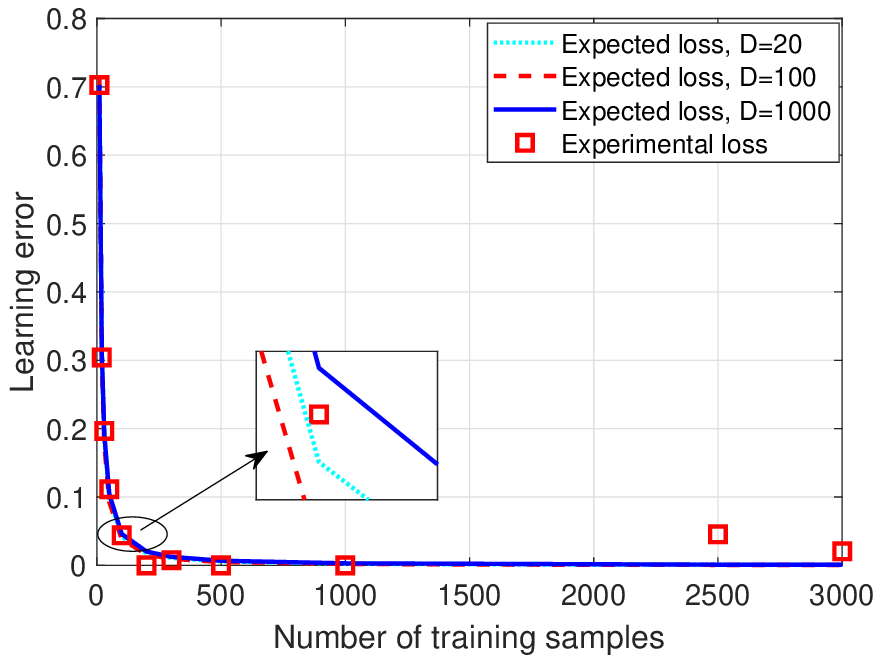} \label{S6-Fig3a}}
	\hfil
	\subfloat[Traffic sign recognition.]{\includegraphics[width=0.32\linewidth]{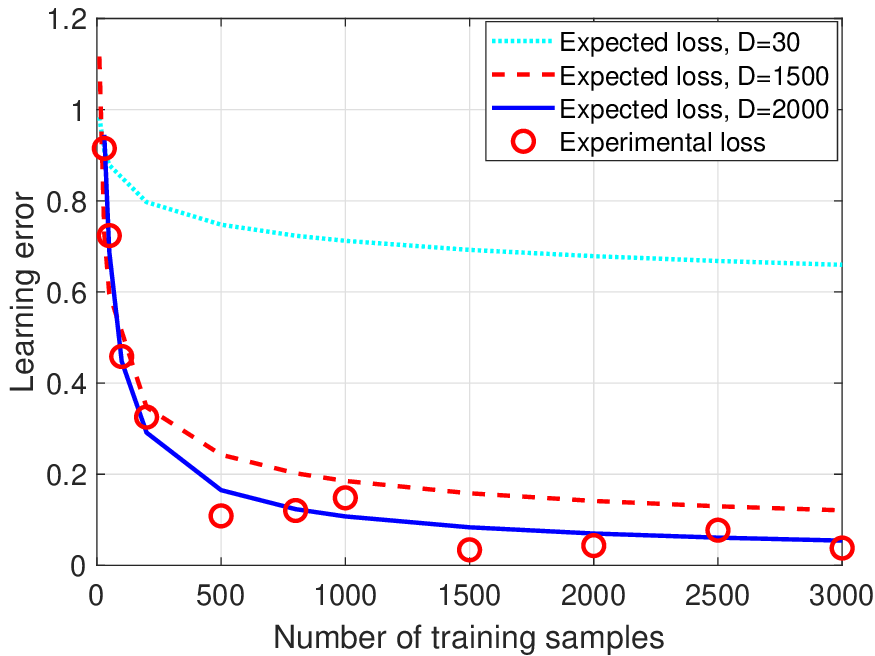} \label{S6-Fig3b}}
	\hfil
	\subfloat[Object detection.]{\includegraphics[width=0.32\linewidth]{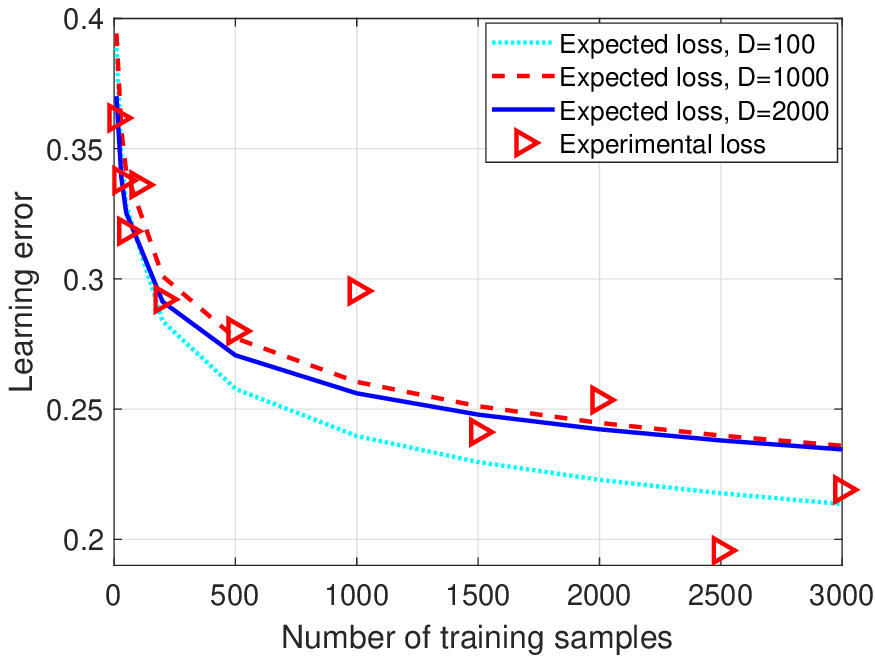} \label{S6-Fig3c}}
	\hfil
	\caption{Learning loss vs. the number of training samples.}
	\label{S6-Fig3}
\end{figure*}
		
\begin{figure*}[t!]
	\centering
	\subfloat[Weather classification.]{\includegraphics[width=0.32\linewidth]{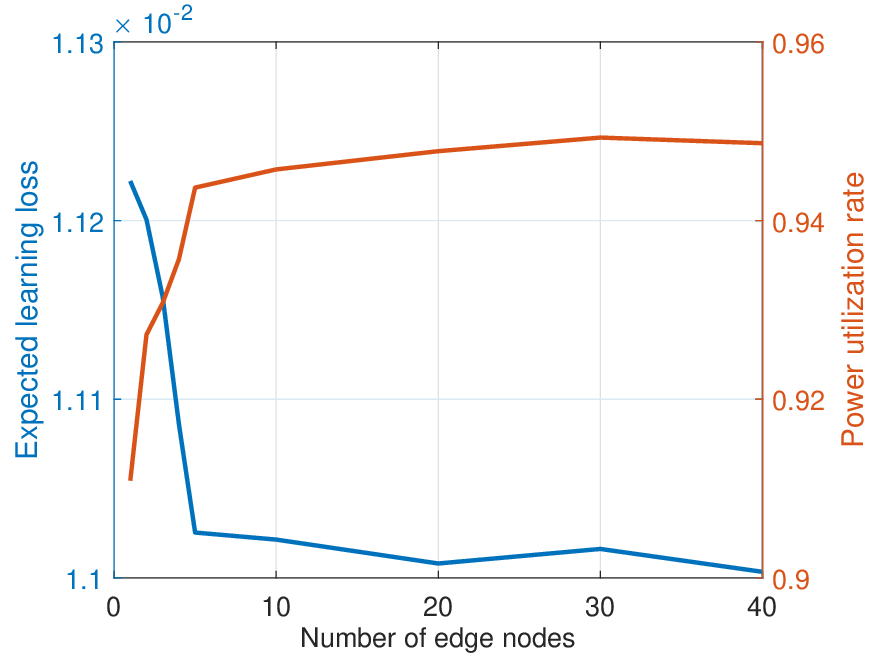} \label{S6-Fig4a}}
	\hfil
	\subfloat[Traffic sign recognition.]{\includegraphics[width=0.32\linewidth]{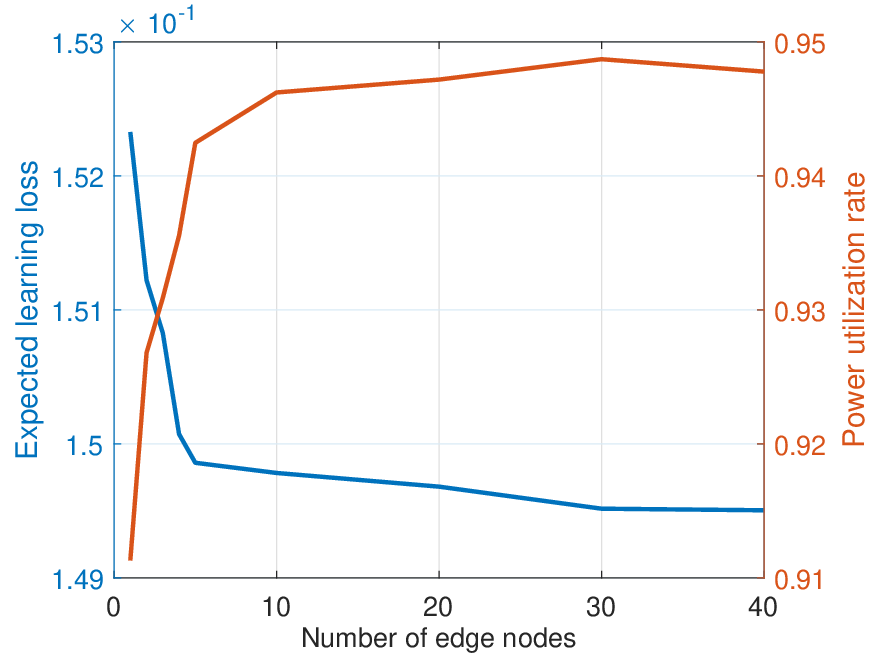} \label{S6-Fig4b}}
	\hfil
	\subfloat[Object detection.]{\includegraphics[width=0.32\linewidth]{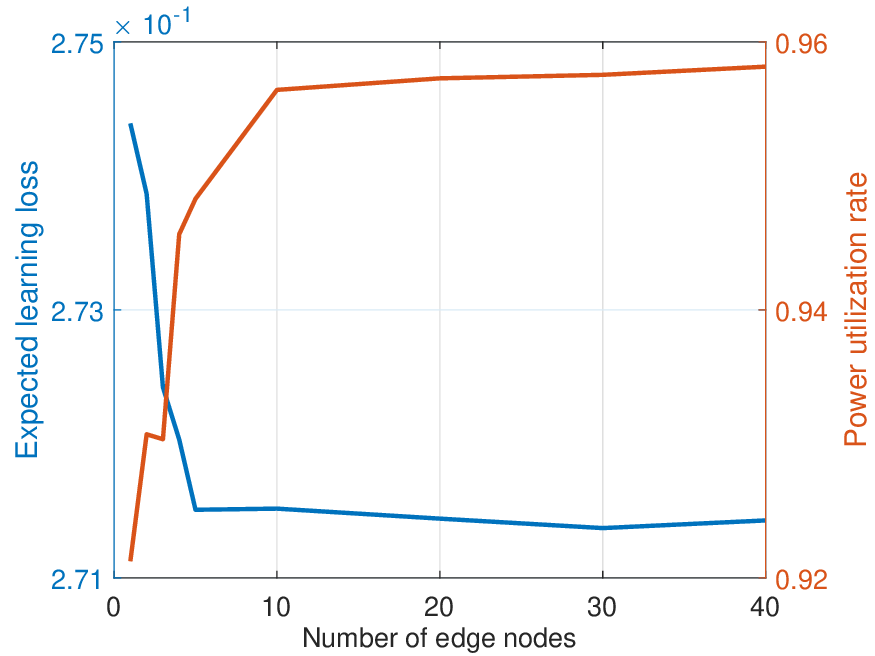} \label{S6-Fig4c}}
	\hfil
	\caption{Expected learning loss/Power utilization ratio vs. the number of edge nodes.}
	\label{S6-Fig4}
\end{figure*}

\section{Simulation Results and Discussion}  \label{S6}
    This section presents simulation results evaluating the performance of the designed algorithms, comparing them with state-of-the-art benchmarks. We primarily consider three perception tasks in autonomous driving \cite{9982368, KouWZ0CNW23}: weather classification from RGB images, traffic sign recognition from RGB images, and object detection from point cloud data. The CarlaFLCAV framework, an open-source autonomous driving simulation platform available online at \url{https://github.com/SIAT-INVS/CarlaFLCAV}, generates all the datasets. The simulation parameter settings are as follows, unless otherwise specified. The size of each RGB image sample is $V_{1} = V_{2} = 0.7$ \si{MB} and that of each point cloud sample is $V_{3} = 1.6$ \si{MB}. The number of historical data samples is $A_{1} = A_{2} = A_{3} = 50$; the number of total IoT devices is $\left| \mathcal{K} \right| = 20$, and the number of edge nodes is $I = 10$. As for the wireless parameters, the Tx time $T = 200$~\si{s}, the number of antennas $N = 4$, the total Tx power $P = 50$ \si{mW}, the noise power $\sigma^{2} = -77$ \si{dBm}, the communication bandwidth $B = 4$ \si{MHz}, the path loss of the $k^{\rm th}$ device $\varrho_{k} = -90$ \si{dB}, and the channel $\bm{h}_{k}$ generated by $\mathcal{C} \mathcal{N} (\bm{0}, \, \varrho_{k} \bm{I})$ are set in this simulation experiments. 
	
We consider two typical benchmark schemes: SRM \cite{10026196} and QoT-Max \cite{10084349}. In particular, the SRM scheme only considers the wireless channel information, ignoring the learning factors. We also simulate two proposed schemes: the MM- and FoM-based distributed algorithms. To evaluate convergence performance, we further define a mean squared error (MSE) as
\begin{align} \label{Eq-MSE}
	{\rm MSE} \triangleq & \left\| \bm{p} (t) - \bm{p} (t - 1) \right\|_{2} + \left\| P_{i} (t) - P_{i} (t - 1) \right\|_{2} \nonumber \\
		&{}+ \left\| \lambda_{i} (t) - \lambda_{i} (t - 1) \right\|_{2} + \left\| \beta (t) - \beta (t - 1) \right\|_{2}.
\end{align}
Next, we present and discuss simulation results on learning error, convergence, and complexity, followed by a case study of autonomous navigation with collision avoidance.

Due to the high computational cost of the end-to-end training--optimization loop, the curves in this section report representative averaged results to illustrate the qualitative energy--learning tradeoff. In additional tests with different random seeds, we observed that the relative performance ordering remains consistent. The proposed framework also supports standard statistical reporting (e.g., mean $\pm$ standard deviation or confidence intervals) when additional computational budget is available.

        \begin{figure*}[t!]
            \centering
            \subfloat[Weather classification.]{\includegraphics[width=0.32\linewidth]{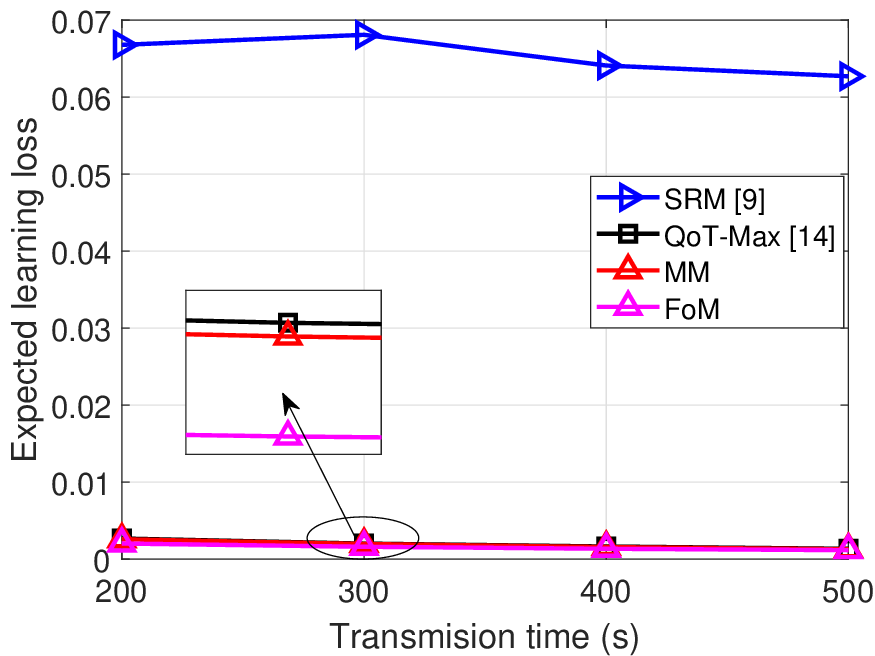} \label{S6-Fig5a}}
            \hfil
            \subfloat[Traffic sign recognition.]{\includegraphics[width=0.32\linewidth]{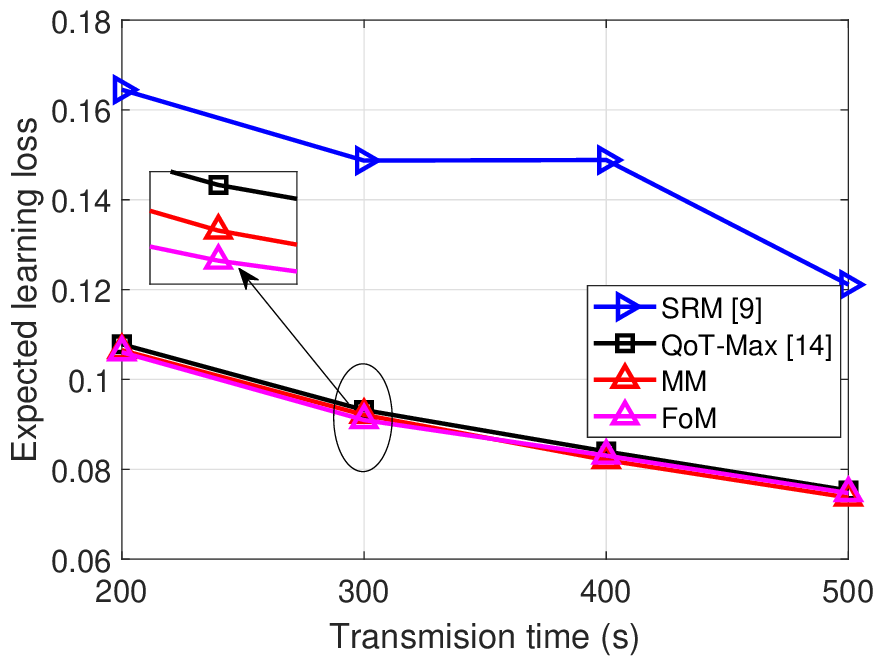} \label{S6-Fig5b}}
            \hfil
            \subfloat[Object detection.]{\includegraphics[width=0.32\linewidth]{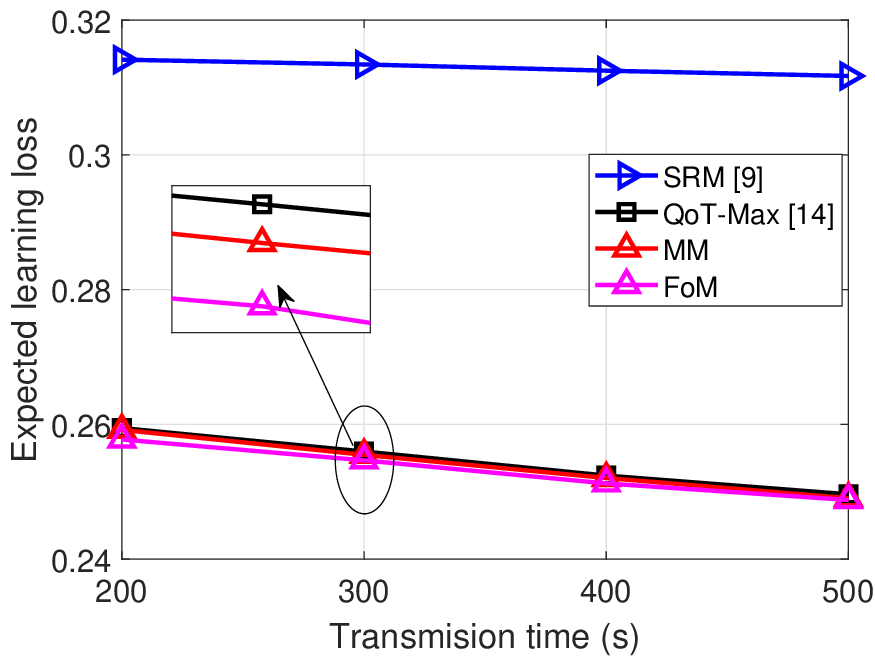} \label{S6-Fig5c}}
            \hfil
            \caption{Learning loss vs. the transmission time.}
            \label{S6-Fig5}
        \end{figure*}

        \begin{figure*}[t!]
            \centering
            \subfloat[Weather classification.]{\includegraphics[width=0.32\linewidth]{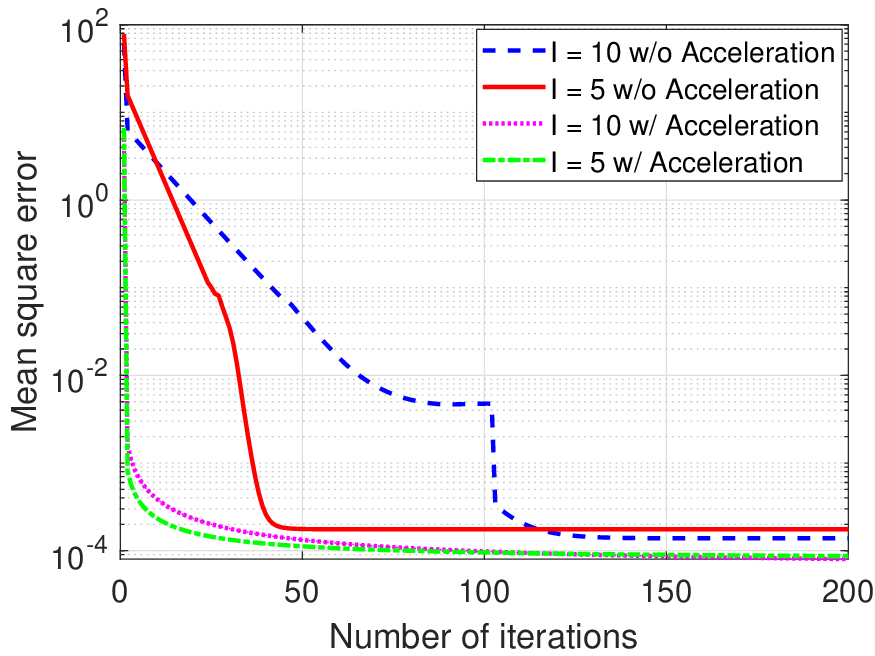} \label{S6-Fig6a}}
            \hfil
            \subfloat[Traffic sign recognition.]{\includegraphics[width=0.32\linewidth]{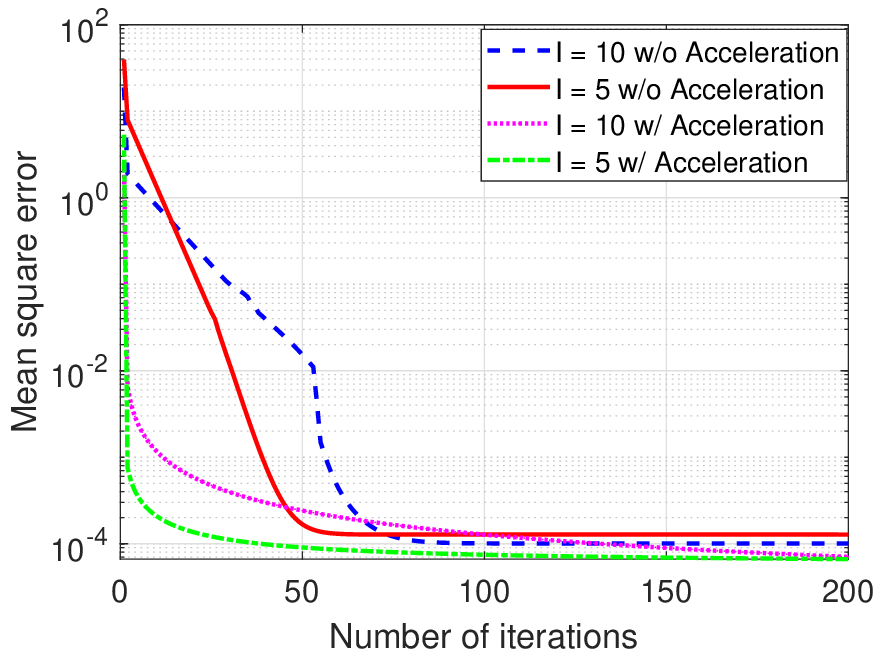} \label{S6-Fig6b}}
            \hfil
            \subfloat[Object detection.]{\includegraphics[width=0.32\linewidth]{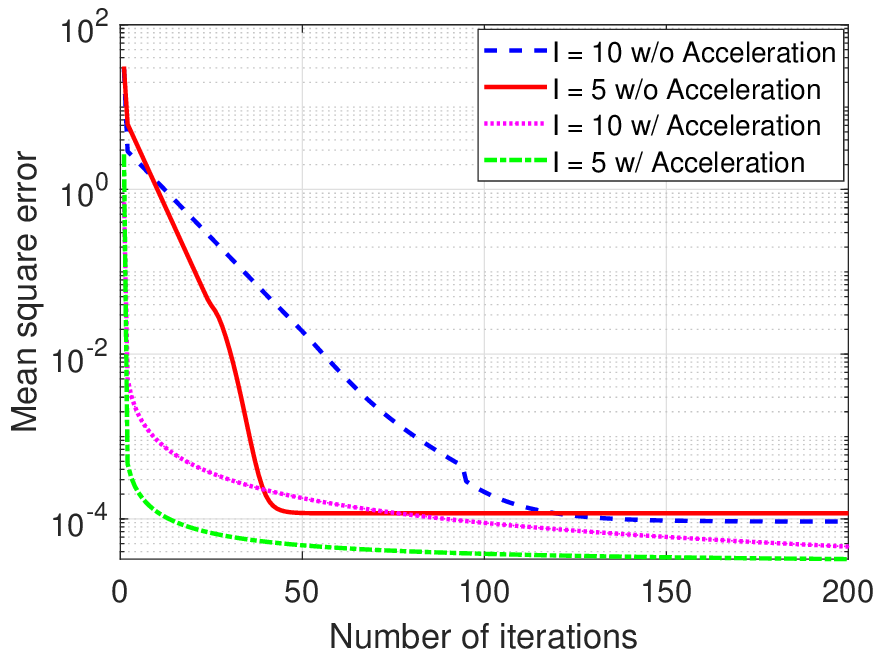} \label{S6-Fig6c}}
            \hfil
            \caption{Mean squared errors (MSE) vs. the number of iterations, with the learning step $\eta = 1 \times 10^{-3}$.}
            \label{S6-Fig6}
        \end{figure*}

\subsection{Learning Error Performance}
    Fig.~\ref{S6-Fig3} depicts the experimental learning loss versus the number of training samples, where the expected learning losses with various training sample sizes are also plotted. Notably, we approximate the expected learning curves across different datasets using the fitting metrics described in Remark~\ref{S3-R2}. Fig.~\ref{S6-Fig3a} (Weather classification) shows that the experimental and expected learning losses decrease with the training samples, and they agree very well. This is because the expected loss converges after a sufficient number of iterations, even when a dedicated training dataset is provided. Moreover, three expected learning curves are approximately similar when the number of samples is set to $D = 20, 100, 1000$. Since the learning model attains competitive performance with small-scale datasets, scaling up the datasets yields diminishing marginal improvements. On the other hand, Figs.~\ref{S6-Fig3b} (Traffic sign recognition) and \ref{S6-Fig3c} (Object detection) show that the expected learning loss significantly deviates from the actual data distribution when the number of training samples is relatively small. The reason is that a small-scale training sample results in high overfitting, and the learning model fails to capture the distribution of the actual data. 
    
    Fig.~\ref{S6-Fig4} demonstrates that the expected learning loss decreases as the number of edge nodes increases, while the power utilization ratio rises. This trend holds across various tasks, including weather classification, traffic sign recognition, and object detection. The primary reason for this improvement is the collaborative resource allocation framework, which enhances the efficiency of wireless resource utilization as more edge nodes are added. Additionally, increasing the number of edge nodes yields a larger sample set, thereby improving learning performance. Notably, the curves converge even with only a few edge nodes, highlighting the effectiveness of our collaborative resource allocation framework.
    
    Fig.~\ref{S6-Fig5} compares the proposed MM- and FoM-based algorithms against benchmark schemes. The results demonstrate that the proposed algorithms outperform traditional SRM algorithms in terms of expected learning loss performance. This advantage arises from the fact that the proposed algorithms focus on minimizing expected learning loss and convergence bounds while optimizing wireless resource allocation based on a learning-oriented principle. This finding aligns with other task-oriented resource allocation schemes \cite{9982368, KouWZ0CNW23, 9606667, 10217150, 9151375}. However, the SRM algorithm exhibits a weak connection between wireless resource allocation and edge learning models. Additionally, the two proposed algorithms show slightly better performance compared to the QoT-Max algorithm. It is noteworthy that the MM-based algorithm performs somewhat less effectively than the FoM-based algorithm. The reason is that the proposed FoM-based algorithm is adaptive to distributed learning and optimization, whereas the MM-based algorithm is designed for centralized learning and optimization.

		\begin{figure*}[t!]
			\centering
			\subfloat[Weather classification.]{\includegraphics[width=0.32\linewidth]{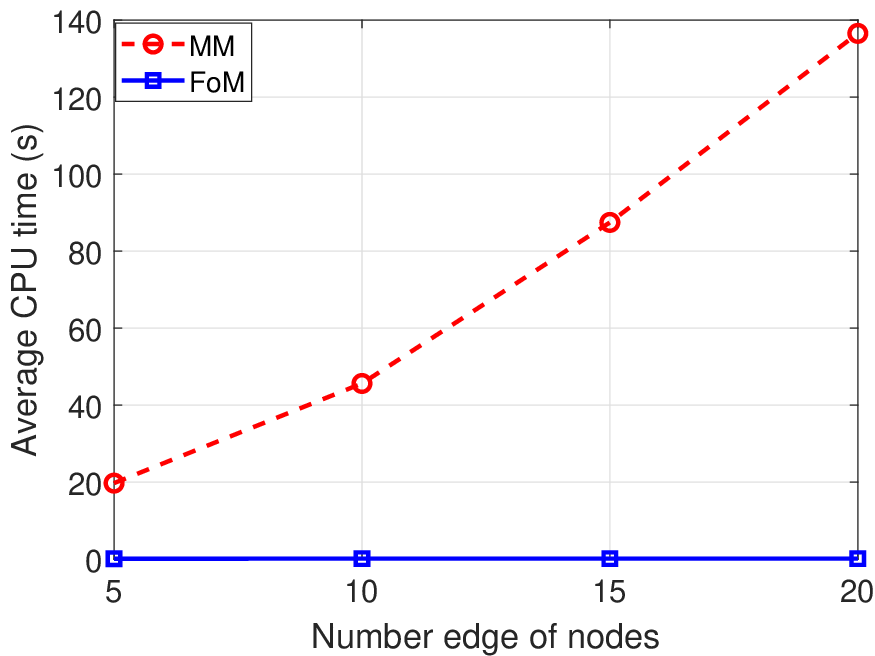} \label{S6-Fig7a}}
			\hfil
			\subfloat[Traffic sign recognition.]{\includegraphics[width=0.32\linewidth]{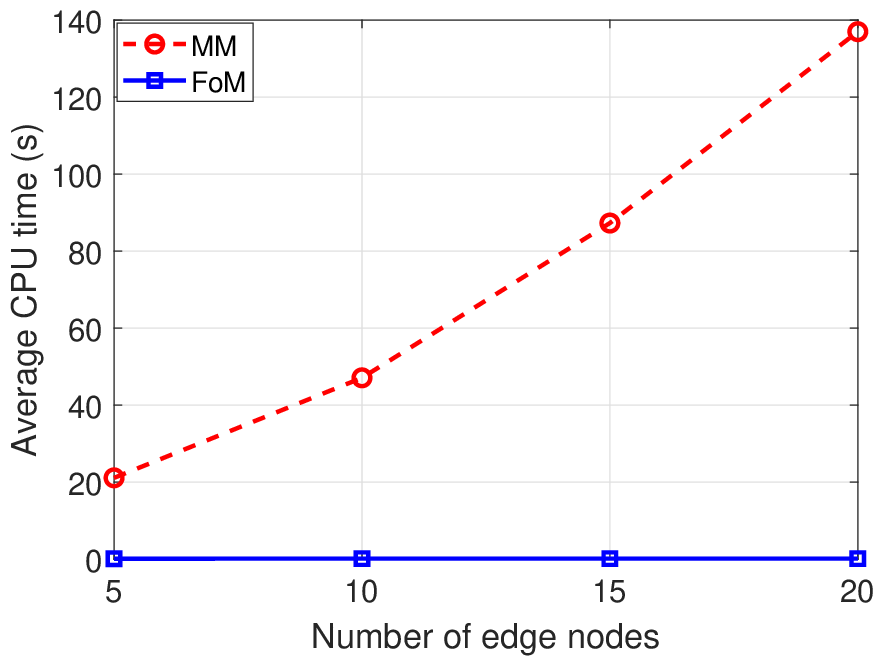} \label{S6-Fig7b}}
			\hfil
			\subfloat[Object detection.]{\includegraphics[width=0.32\linewidth]{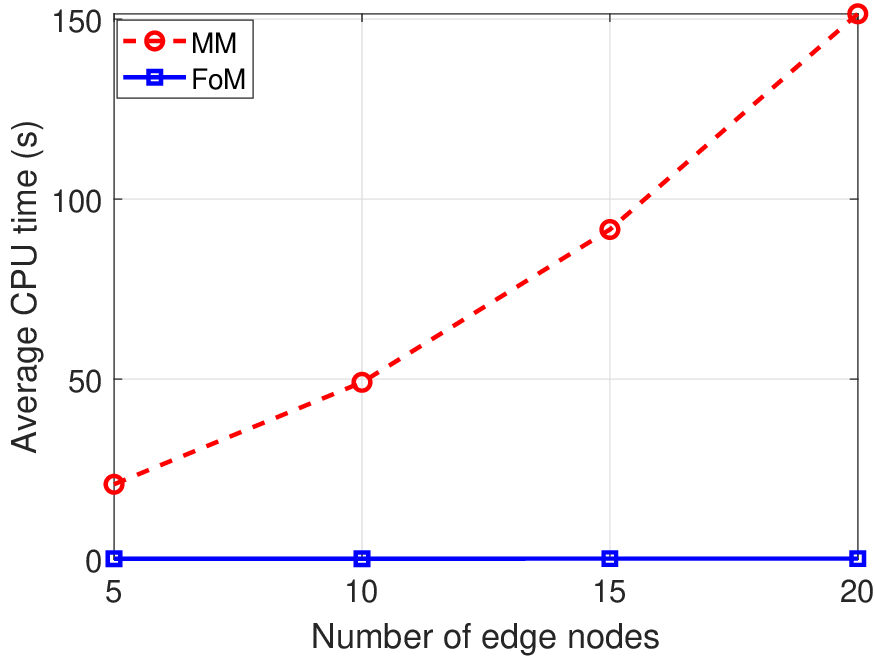} \label{S6-Fig7c}}
			\hfil
			\caption{Average CPU time vs. the number of edge nodes, with $|\mathcal{K}_{1}| = |\mathcal{K}_{2}| = \cdots = |\mathcal{K}_{I}| = 10$.}
			\label{S6-Fig7}
		\end{figure*}

\subsection{Convergence Performance and Complexity Analysis}    
    To evaluate the convergence performance of the algorithms, Fig.~\ref{S6-Fig6} shows the MSE, defined by \eqref{Eq-MSE}, as a function of the number of iterations. On the one hand, we observe from Fig.~\ref{S6-Fig6} that the proposed algorithm without a momentum acceleration term converges faster when $I = 5$ than when $I = 10$. This trend is consistent across various tasks, including weather classification, traffic sign recognition, and object detection. The difference in convergence rates can be attributed to the greater constraint relaxation required when deploying distributed networks with $I = 10$ rather than $I = 5$. As the number of edge nodes increases, the convergence rate of the proposed algorithm gradually decreases, making it more challenging to deploy large-scale distributed networks.
        	
    On the other hand, the proposed algorithm, incorporating a momentum acceleration term, exhibits a steeper decrease, as evident in Fig.~\ref{S6-Fig6}. This improvement occurs because the momentum acceleration term reduces the convergence rate from $\mathcal{O}(1/t)$ to $\mathcal{O}(1/t^{2})$. Furthermore, this term helps prevent the algorithm from becoming trapped in local oscillations. Therefore, adding momentum makes the algorithm more suitable for deploying large-scale distributed networks. Overall, we conclude that an increase in the dynamic growth of edge nodes leads to a decrease in the convergence of the collaborative model, as shown in Figs.~\ref{S6-Fig6}(a)-(c). This decline is attributed to reduced distributed information and heterogeneous communication, leading to significant fluctuations in the accuracy of the collaborative model.
    
    Fig.~\ref{S6-Fig7} illustrates the relationship between average CPU time and the number of edge nodes. The computation time of the distributed FoM-based algorithm remains relatively stable as the number of edge nodes increases. Specifically, Figs.~\ref{S6-Fig7a}-\ref{S6-Fig7c} indicate that the average CPU times for three different tasks are approximately $0.15$, $0.1$, and $0.18$ seconds, respectively. This stability is attributed to the FoM algorithm's parallel nature, which enables easy deployment across distributed networks and thereby reduces computational complexity. In contrast, the average CPU time for the centralized MM-based algorithm increases significantly as the number of edge nodes grows.

	Regarding computational complexity, the centralized MM-based algorithm involves $I + 2K + 2$ optimization variables and requires the CVX or Mosek package to solve the convex optimization problem. As a result, the worst-case complexity of this algorithm can be computed as $\mathcal{O} \left( (I + 2K + 2)^{3.5} \right)$. On the other hand, the first-order FoM algorithm involves $2I + 2K + 1$ optimization variables and has a linear complexity of $\mathcal{O}(2I + 2K + 1)$. Additionally, since most of the variables can be updated in a distributed manner, the complexity of the first-order FoM algorithm can be approximated as $\mathcal{O} (K / I)$. Overall, we present the complexity results of three benchmark algorithms and two proposed algorithms in the first column of Table~\ref{Table2}.
    
	In addition to examining computational complexity, Table~\ref{Table2} compares three other factors: distributed capability, online capability, and learning performance. The proposed FoM-based algorithm is effective for distributed networks and online learning due to its low computational complexity, strong distribution, and online capabilities. It also demonstrates high learning performance, as illustrated in Fig.~\ref{S6-Fig5}. In contrast, the MM-based algorithm is suited for small- to medium-scale IoT networks, thanks to its high learning performance.    
	
	In practice, the proposed online optimization is executed at the edge orchestrator once per control interval (spanning multiple SGD steps) and can be warm-started across consecutive intervals. The control signaling for distributing power plans consists of only a few scalars per node and can be piggybacked on existing control messages. In our simulations, we focus on the energy--learning tradeoff and thus do not explicitly model the solver runtime and signaling cost; nevertheless, these overheads are typically negligible compared with training and payload communication in large-scale IoT networks, and can be further reduced by using longer control intervals or quantized power levels in small-scale deployments.

   		\begin{figure*}[t!]
		  	\centering
		  	\subfloat[The expected loss.]{\includegraphics[width=0.24\linewidth]{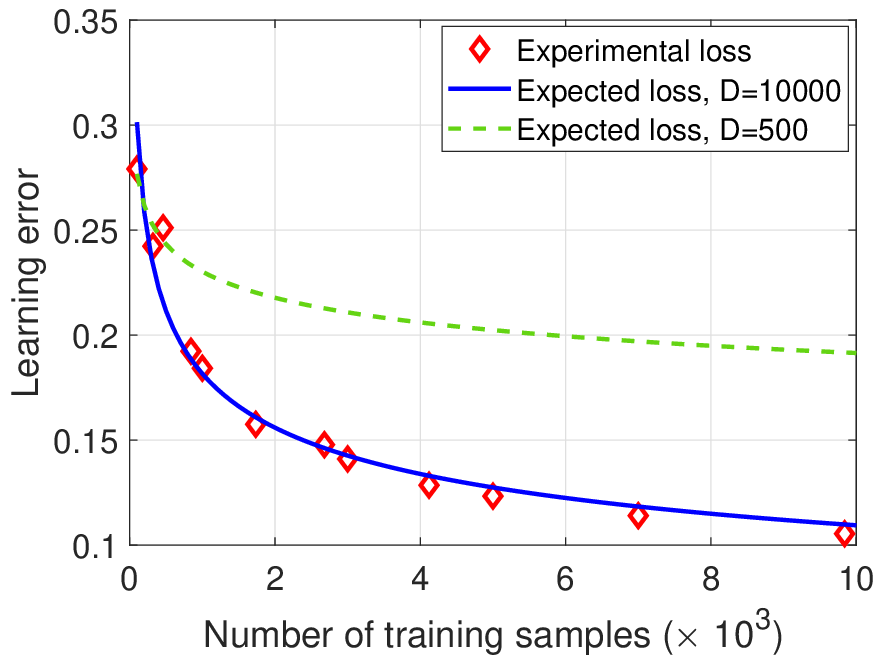} \label{S6-Fig8a}}
		  	\hfil
		  	\subfloat[The learning rate.]{\includegraphics[width=0.24\linewidth]{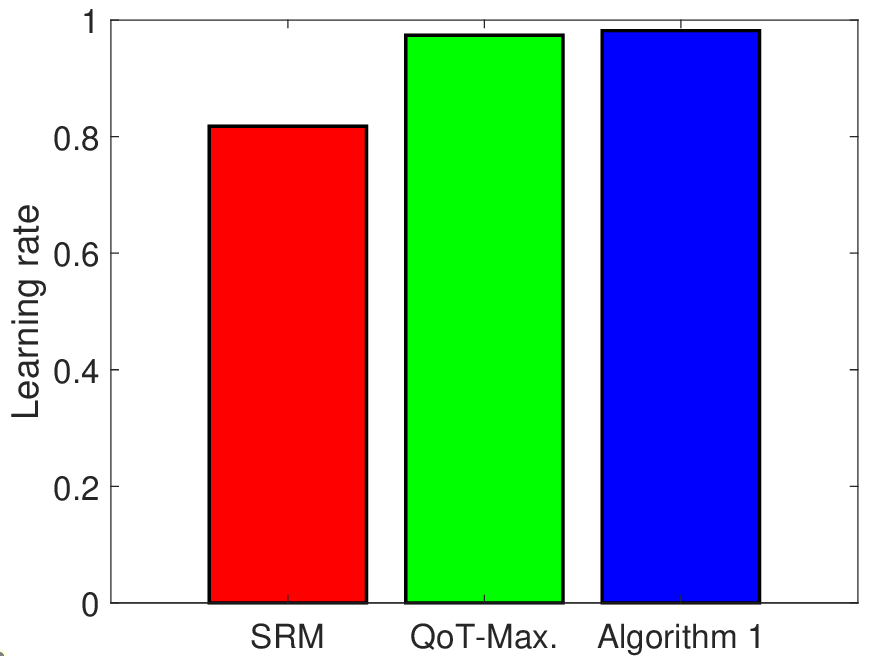} \label{S6-Fig8b}}
		  	\hfil
		  	\subfloat[The average collision rate.]{\includegraphics[width=0.24\linewidth]{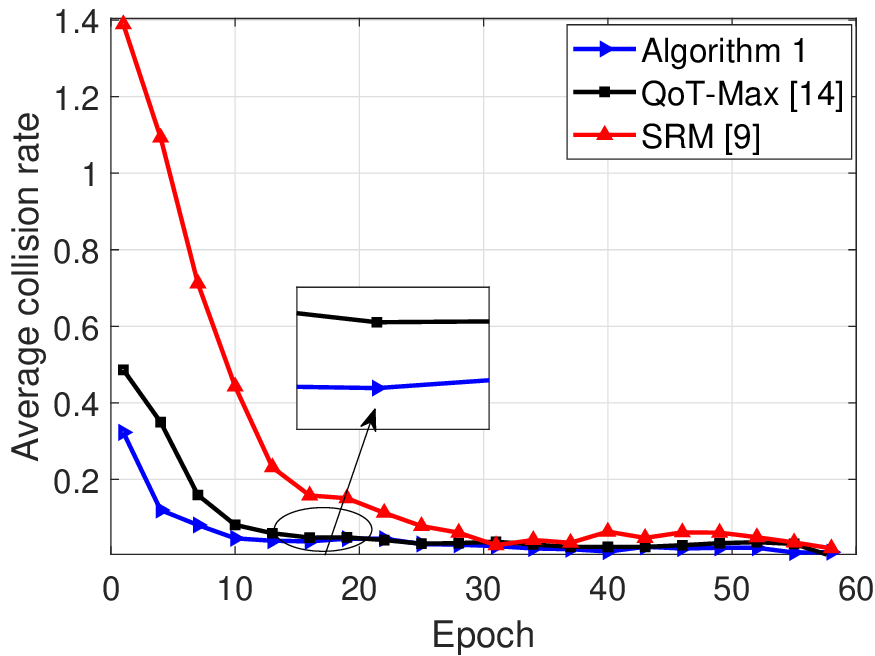} \label{S6-Fig8c}}
		  	\hfil
		  	\subfloat[The average goal rate.]{\includegraphics[width=0.24\linewidth]{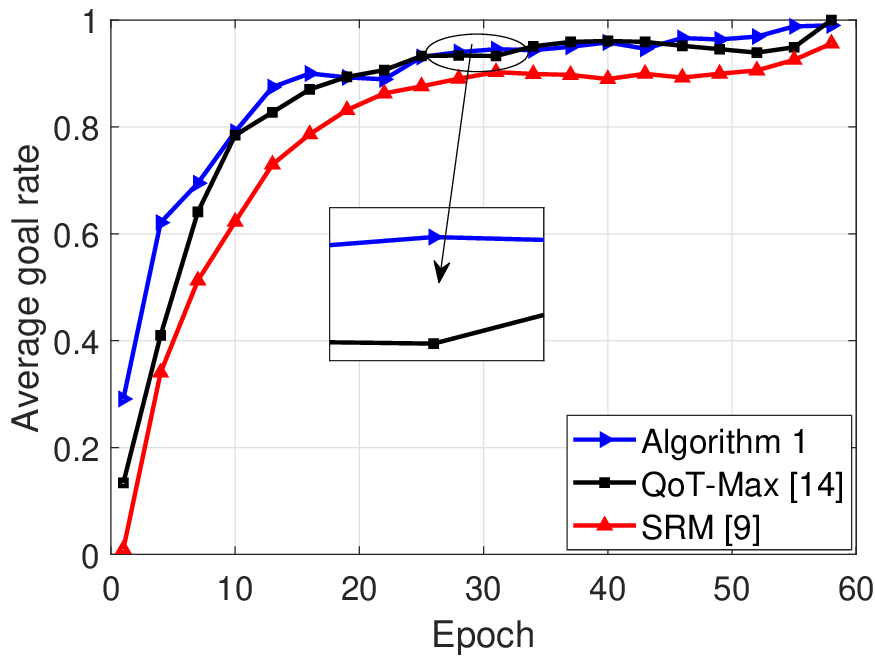} \label{S6-Fig8d}}
		  	\hfil
		  	\caption{The training behavior in autonomous navigation.}
		  	\label{S6-Fig8}
		  \end{figure*}   

\begin{table}[t!]
	\tiny
	\setstretch{1.0}
	\centering
	\caption{Performance comparison.}
	\vspace{-5pt}
	\setlength{ \arraycolsep }{-0.2em}
	\begin{threeparttable}[!t]
		\begin{tabular}{!{\vrule width1.2pt}c !{\vrule width1.2pt} c !{\vrule width1.2pt}c !{\vrule width1.2pt}c !{\vrule width1.2pt}c !{\vrule width1.2pt}}
			\Xhline{1.2pt}
			\textbf{Algorithm} & \textbf{Complexity} & \textbf{Distribution}\tnote{a} & \textbf{Online} & \begin{tabular}[c]{@{}c@{}}\textbf{Learning} \\ \textbf{Performance} \end{tabular}\\
			\Xhline{1.0pt}
			SRM \cite{10026196} & $\mathcal{O} \left( ( K - 1 )^{7} \right)$ &  \XSolidBrush & \XSolidBrush & low \\
			\hline
			QoT-Max \cite{10084349} & $\mathcal{O} \left( K  / I \right)$ & \Checkmark & \XSolidBrush & high \\
			\hline
			ADMM-based LCPA \cite{10120724} & $\mathcal{O} \left( ( K^{2} + K ) / I \right)$ & \Checkmark & \XSolidBrush & high \\
			\hline
			MM-based Proposal & $\mathcal{O} \left( ( I + K^{2} + K )^{3.5} \right)$ &  \XSolidBrush & \XSolidBrush & high \\
			\hline
			FoM-based Proposal & $\mathcal{O} \left( K / I \right)$ & \Checkmark & \Checkmark & high \\
			\Xhline{1.0pt}
		\end{tabular}
		\label{Table2}
		{\scriptsize
			\begin{tablenotes}
				\item[a] The ``\Checkmark'' indicates a functionality supported, whereas the ``\XSolidBrush" not supported.
			\end{tablenotes}}
	\end{threeparttable}
\end{table}

    \subsection{Case Study: Autonomous Navigation with Collision Avoidance}

	We evaluate the proposed algorithms through a case study on vehicle navigation with shape-based collision avoidance \cite{10036019, 10758752}. Simulations are conducted on the IR-SIM platform \cite{9645287}, an open-source navigation environment (\url{https://github.com/reiniscimurs/DRL-robot-navigation}). Unless otherwise specified, we deploy $\left|\mathcal{K}\right|=20$ sensors on a single edge server with wireless settings consistent with earlier experiments. Each LiDAR packet has a size of $V = 1.6$ \si{MB}, and the initial allocation is $ A = 500$. LiDAR point-cloud datasets are generated via CarlaFLCAR and used for training within IR-SIM, while the trained models are validated in CARLA for visualizing obstacle avoidance.

	To quantify training performance, in addition to learning loss, we define three metrics. The learning rate $r_{l}$ reflects the ratio of learning time to total time, the average collision rate $r_{c}$ measures collisions per obstacle, and the average goal rate $r_{g}$ counts successful arrivals per trial:
\begin{equation} \label{S6-Eq24}
	r_{l} = \dfrac{T_{l}}{T_{l} + T_{o}}, \quad 
	r_{c} = \dfrac{N_{c}}{N_{c} + N_{a}}, \quad
	r_{g} = \dfrac{N_{g}}{N_{g} + N_{u}}, 
\end{equation}
where $r_{l}, r_{c}, r_{g}$ denote the learning rate, collision rate, and goal rate, respectively; $T_{l}, N_{c}, N_{g}$ are training time, number of collisions, and number of goals, while $T_{o}, N_{a}, N_{u}$ are optimization time, number of avoidances, and failed goals.

Fig.~\ref{S6-Fig8} illustrates the training behavior. In Fig.~\ref{S6-Fig8a}, the expected learning loss decreases with more samples, but small-scale fitting ($500$ samples) diverges from actual LiDAR data, while large-scale fitting ($10000$ samples) aligns well. Fig.~\ref{S6-Fig8b} shows that Algorithm~\ref{S5-AM1} and QoT-Max maintain learning rates above $0.95$, outperforming SRM ($<0.85$) due to their lower computational complexity. Fig.~\ref{S6-Fig8c} compares collision rates: Algorithm~\ref{S5-AM1} consistently achieves the lowest, benefiting from its task-oriented optimization that collects more samples. Fig.~\ref{S6-Fig8d} shows goal rates, where Algorithm~\ref{S5-AM1} again surpasses QoT-Max and SRM, which suffer from insufficient~data.  

In the CARLA navigation task, the learning loss measures the policy’s training objective (e.g., prediction/behavioral-cloning error) and serves as a proxy for control quality. As training proceeds, lower loss generally leads to more stable driving behavior, consistent with the improved task-level outcomes in Fig.~8, such as a higher goal rate $r_g$ and a lower collision rate $r_c$.

Fig.~\ref{S6-Fig9} visualizes autonomous navigation outcomes across different algorithms. The top-right subfigures show the final navigation states, while the main panels display the corresponding trajectories. Black and red vehicles in the top-right subfigures denote autonomous agents and obstacles, respectively. In the main panels, gray regions represent roads, red dashed lines indicate trajectories, blue polygons denote obstacles, and white circular markers mark the positions of autonomous vehicles.

Performance comparison across algorithms reveals distinct behaviors. For Algorithm~\ref{S5-AM1} (Fig.~\ref{S6-Fig9a}), the left subfigure shows stable navigation, and the right demonstrates successful goal completion under disturbances, enabled by efficient resource utilization and extensive training sample collection. For QoT-Max (Fig.~\ref{S6-Fig9b}), the left subfigure depicts an unstable but partially effective trajectory, while the right shows a right-turn failure due to divergence between small-scale initial data and the accurate LiDAR distribution. SRM (Fig.~\ref{S6-Fig9c}) exhibits collisions and right-turn failures, as its non-task-oriented design limits data collection and learning performance.

	   \begin{figure}[t!]
	    	\centering
	    	\subfloat[The proposed Algorithm~\ref{S5-AM1}.]{
		\includegraphics[width=0.5\linewidth]{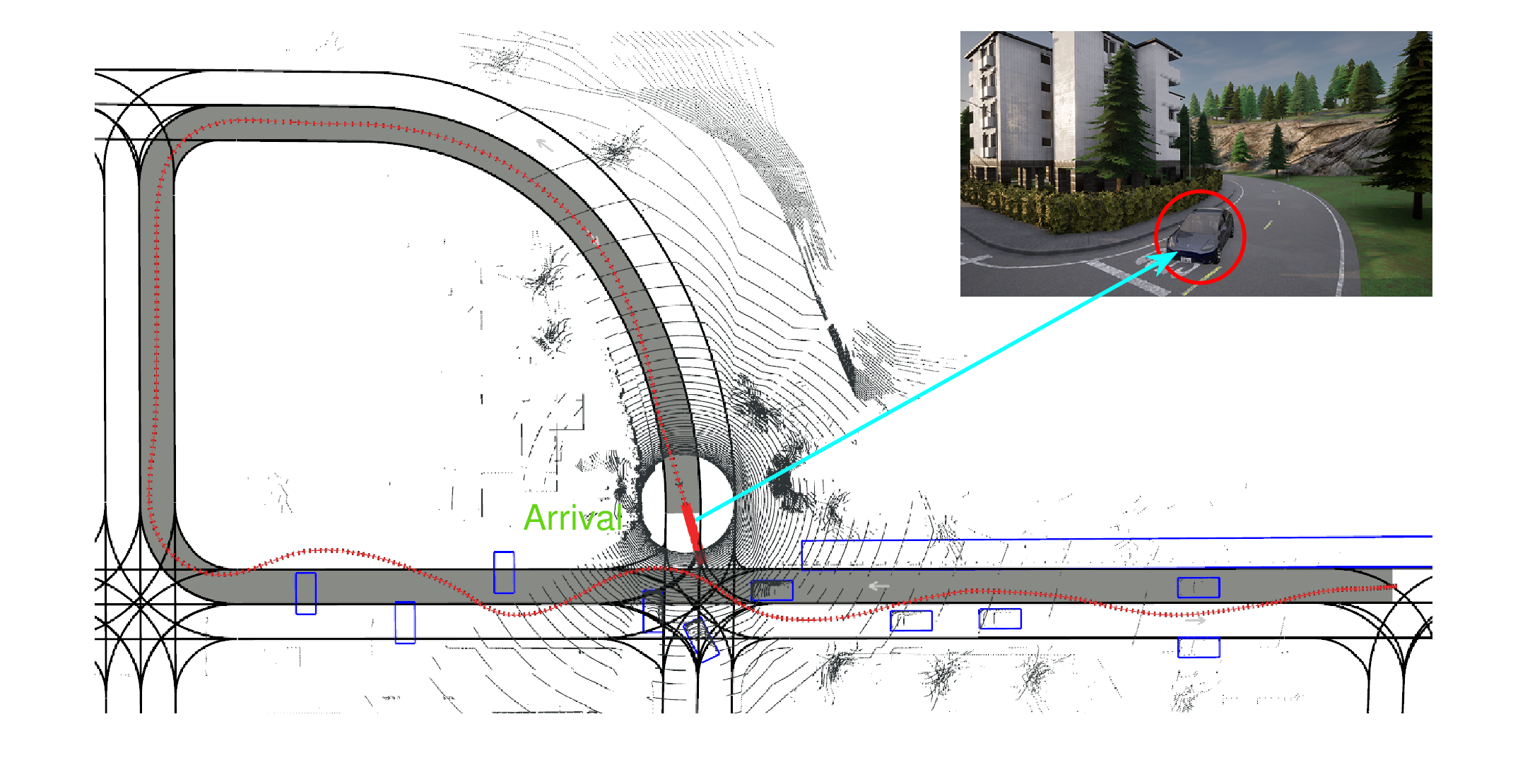} 
		\includegraphics[width=0.5\linewidth]{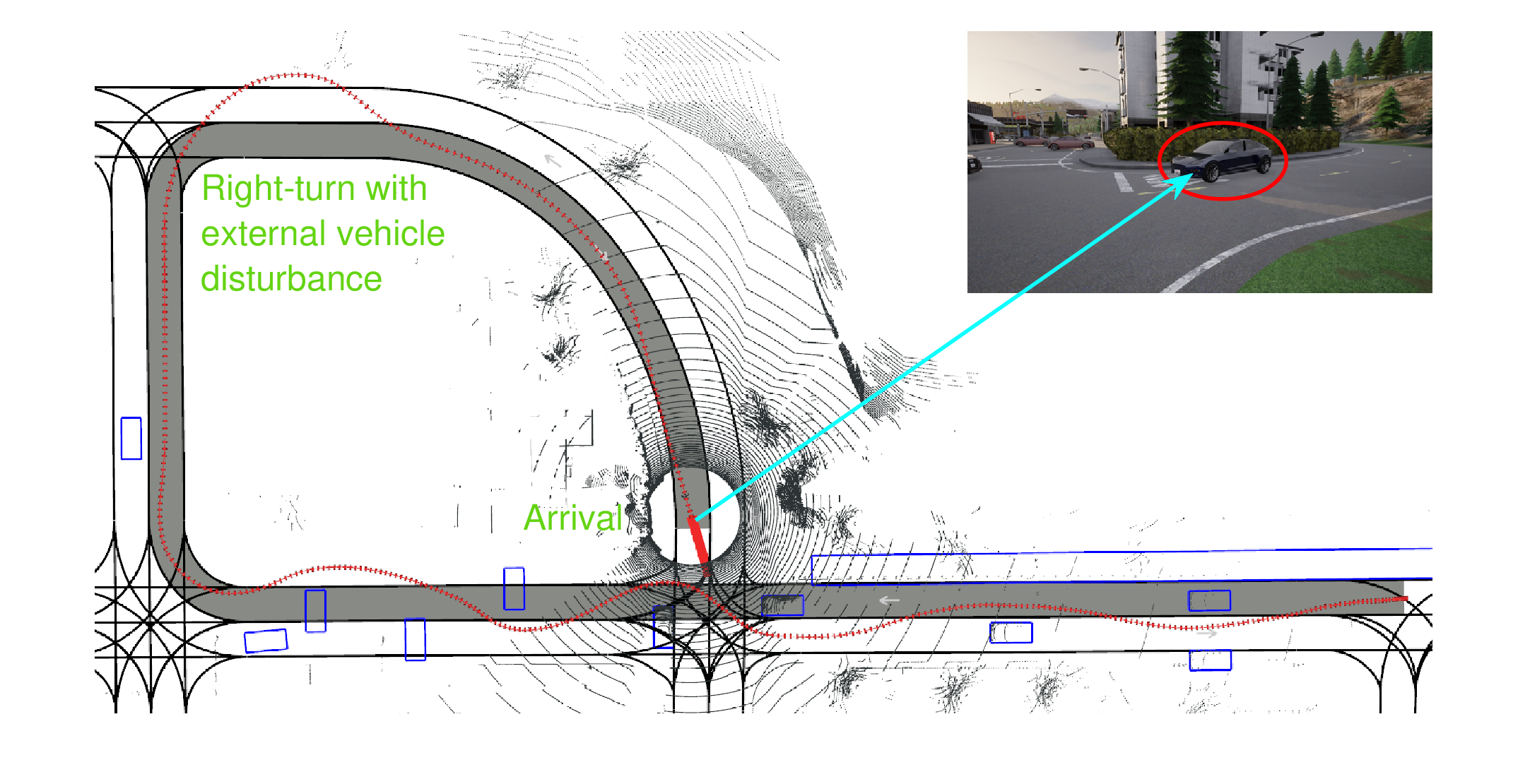} \label{S6-Fig9a}}
	    	\hfil	
	    	\subfloat[The QoT-Max algorithm \cite{10084349} with small-scale initialized data.]{
		\includegraphics[width=0.5\linewidth]{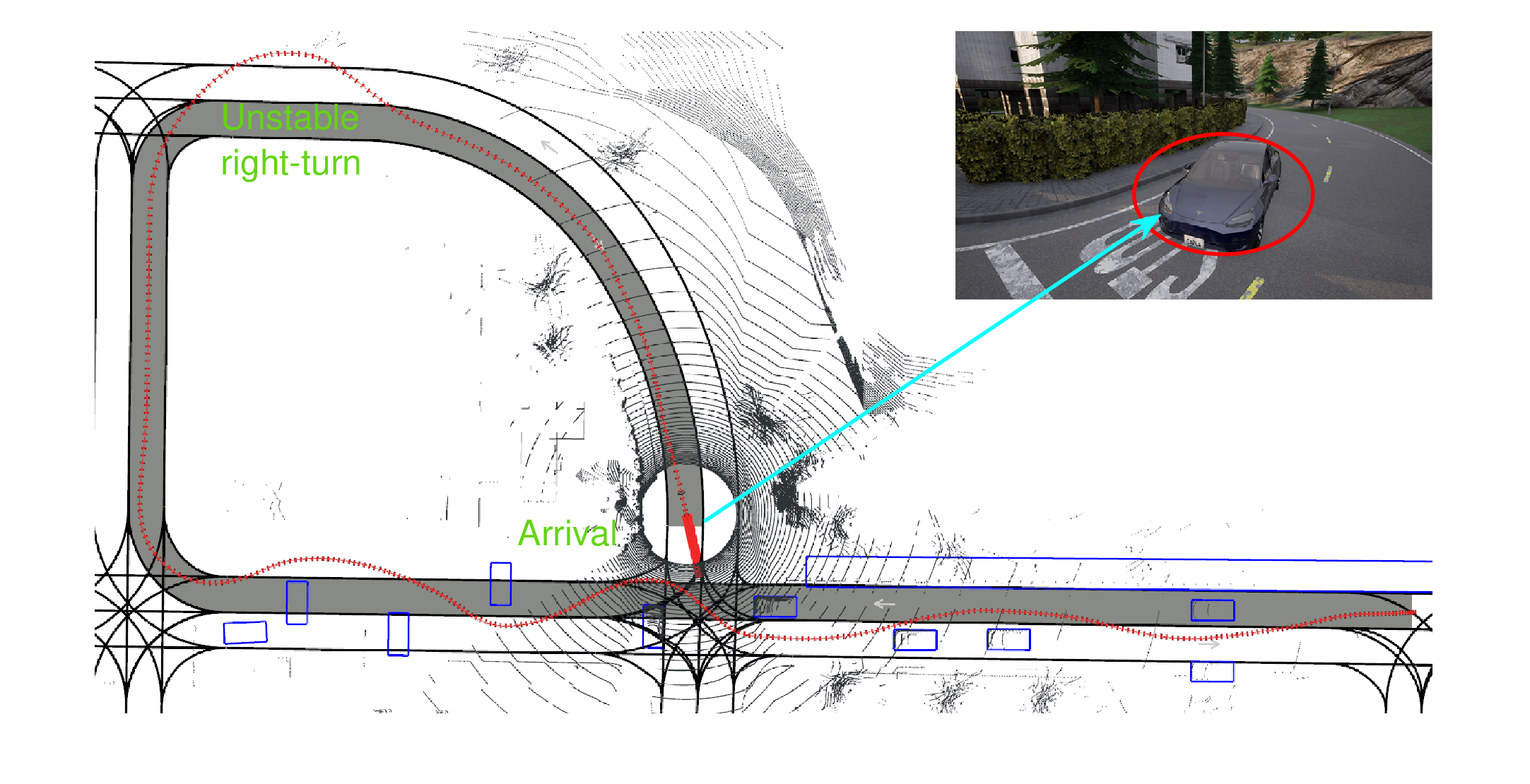} 
	    	\includegraphics[width=0.5\linewidth]{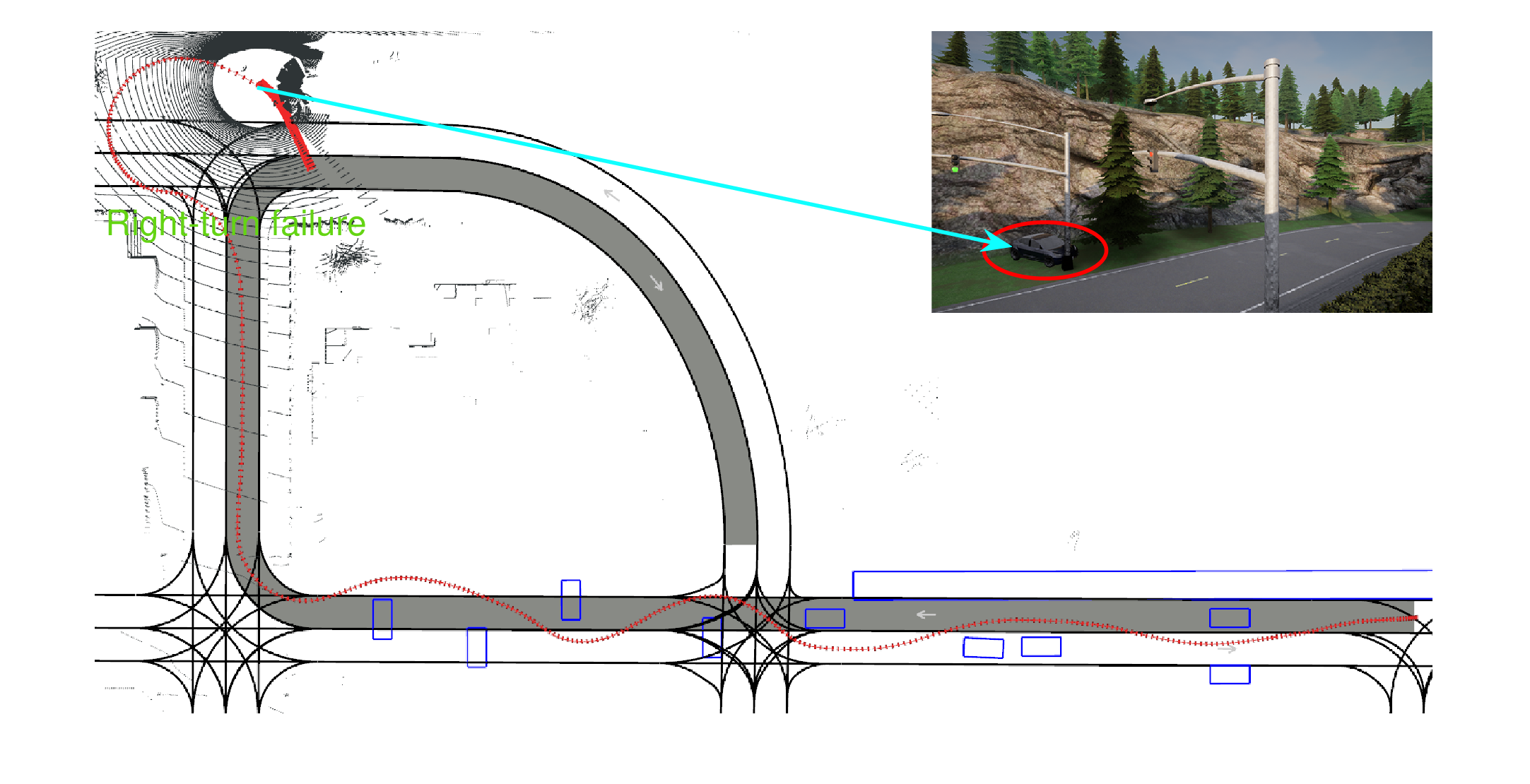} \label{S6-Fig9b}}
	    	\hfil		
	    	\subfloat[The SRM algorithm \cite{10026196} without task-oriented communication.]{
		\includegraphics[width=0.5\linewidth]{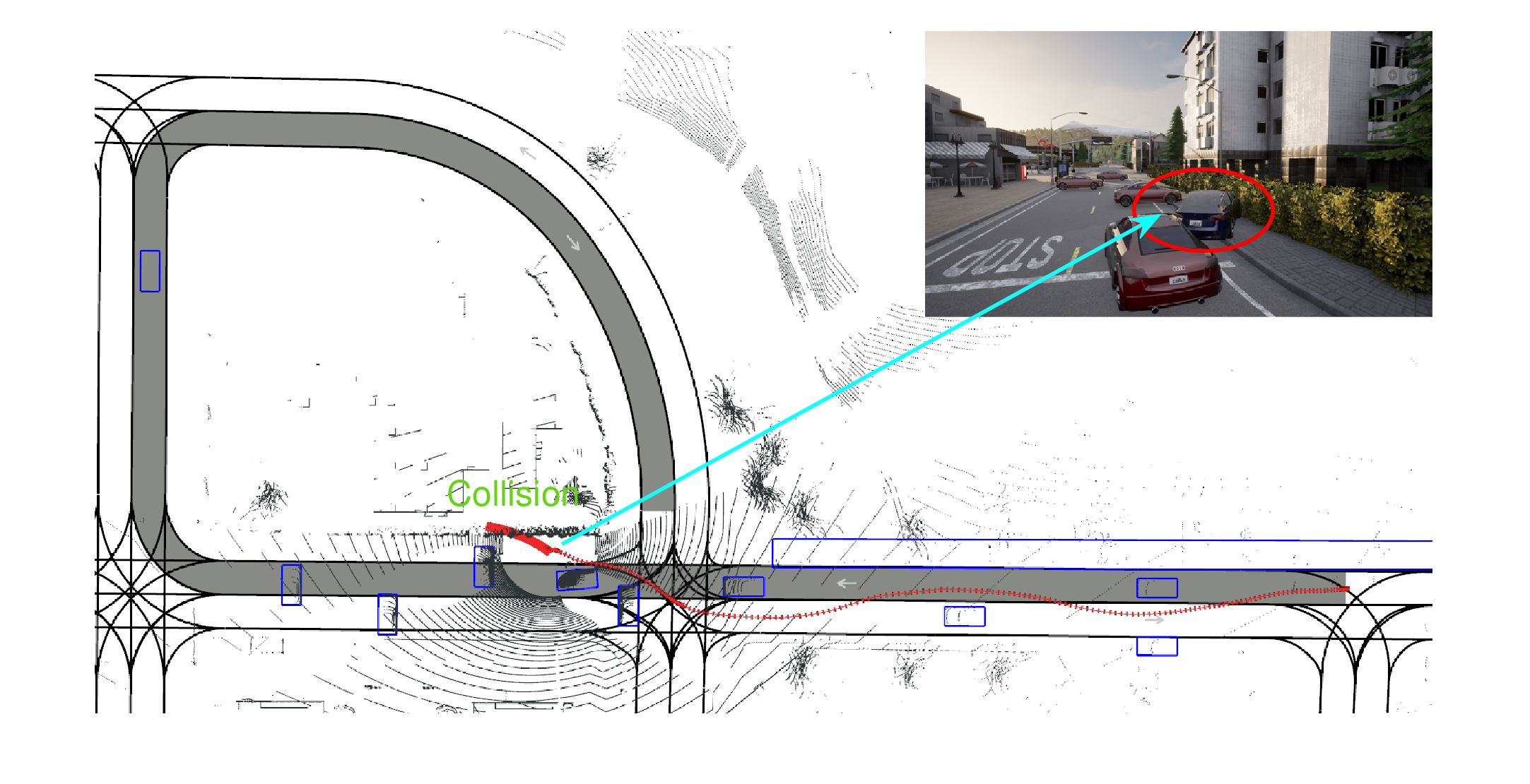} 
	    	\includegraphics[width=0.5\linewidth]{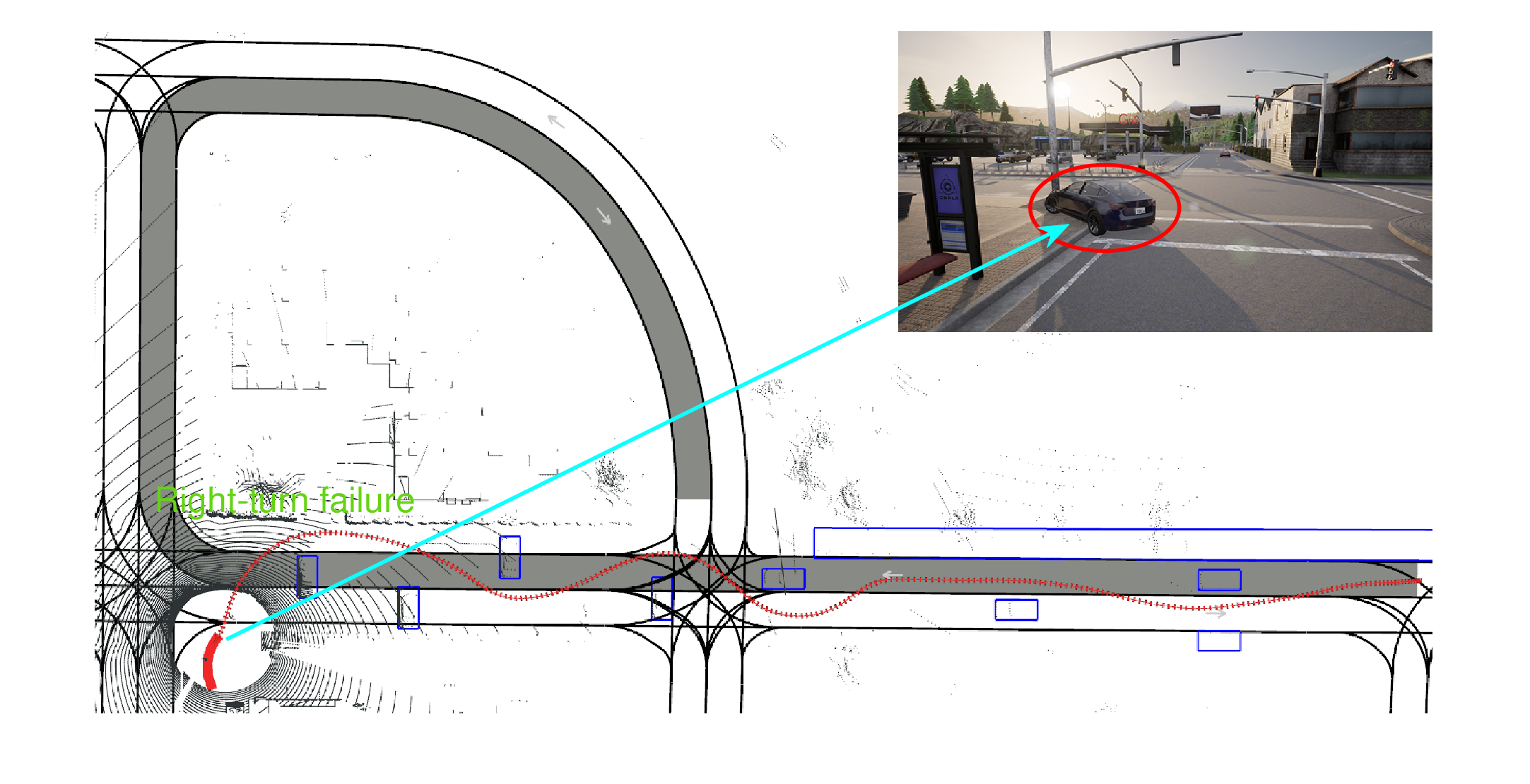} \label{S6-Fig9c}}
	    	\hfil
	\caption{Visualization of autonomous navigation in the CARLA simulator. Each subfigure pair displays the navigation trajectories (main panels) and final states (top-right panels) of the autonomous vehicles. Black and red vehicles indicate agents and obstacles, gray regions represent roads, red dashed lines show trajectories, blue polygons denote obstacles, and white circular markers indicate vehicle positions.}
	    	\label{S6-Fig9}
	    \end{figure}

\section{Conclusion} \label{S7}	
	This paper has proposed an energy-efficient federated edge learning model for processing small-scale datasets in large-scale IoT networks. To address low resource utilization, a collaborative power allocation framework was developed to align with distributed network structures. An expected learning loss and a stochastic online learning algorithm were introduced to establish the relationship between dataset size and learning performance, enabling effective small-scale data collection. Additionally, a highly distributed algorithm was designed to solve large-scale optimization problems efficiently. Extensive experiments and case studies demonstrated that the proposed model significantly improves resource utilization and learning performance, supporting efficient distributed processing with low computational complexity. The model also showed scalability and effectiveness in autonomous driving and navigation scenarios. Future work includes extending the framework to multi-task learning in multi-modal environments.

\appendices
\section{Proof of Theorem~\ref{S3-TM1}} \label{App-Thm1}

We begin by relaxing $F (\bm{w} (t + 1))$ as follows:
		\begin{align} \label{SA-EQA1}
			F ( \bm{w} (t + 1) ) 
			& \leq F ( \bm{w} (t) ) + \nabla F ( \bm{w} (t) )^{T} ( \bm{w} (t + 1) - \bm{w} (t) ) \nonumber \\
			& \quad {}+ \dfrac{L}{2} \| \bm{w} (t + 1) - \bm{w} (t) \|^{2},
		\end{align}
	where \eqref{SA-EQA1} is derived from Assumption~\ref{S3-AS1}. Taking the expectation on both sides of \eqref{SA-EQA1} yields
		\begin{align}
			& \mathbb{E} \left( F ( \bm{w} (t + 1) ) \right) \nonumber \\
			& \leq \mathbb{E} \Big( F ( \bm{w} (t) ) + ( \bm{w} (t + 1) - \bm{w} (t) )^{T} \nabla F ( \bm{w} (t) ) \nonumber \\
			& \quad  {}+ \dfrac{L}{2} \| \bm{w} (t + 1) - \bm{w} (t + 1) \|^{2} \Big) \nonumber \\
			& \leq \mathbb{E} \left( F ( \bm{w} (t) ) \right) - \dfrac{1}{ 2 L } \| \nabla F ( \bm{w} (t) ) \|^{2} + \dfrac{1}{ 2 L } \mathbb{E} ( \| \bm{o} (t) \|^{2} ), \label{SA-EQA2b}
		\end{align}
	where \eqref{SA-EQA2b} is obtained by \eqref{S2-EQ3a}-\eqref{S2-EQ3b}. Specifically, $\bm{o} (t)$ is defined as
		\begin{equation} \label{SA-EQA2c}
			\bm{o} (t) \triangleq \dfrac{ D - | \mathcal{X} (t) | }{ D | \mathcal{X} (t) | } \sum_{ i \in \mathcal{X} (t) } \nabla f_{i} ( \bm{w} (t) ) - \dfrac{1}{D} \sum_{ i \in \mathcal{X}^{C} (t) } \nabla f_{i} ( \bm{w} (t) ),
		\end{equation}
	where $\mathcal{X}^{C} (t)$ denote the complement of $\mathcal{X} (t)$. By \eqref{SA-EQA2c}, we can derive $\mathbb{E} ( \| \bm{o} (t) \|^{2} )$ as follows
		\begin{align}
			\mathbb{E} ( \| \bm{o} (t) \|^{2})
			& \leq \mathbb{E} \left( \left[ \left( \dfrac{ D - | \mathcal{X} (t) | }{ D | \mathcal{X} (t) | } \right) \sum_{ i \in \mathcal{X} (t) } \| \nabla f_{i} ( \bm{w} (t) ) \| \right. \right. \nonumber \\
			&\qquad \left. \left. {}+ \dfrac{1}{ | \mathcal{X}^{C} (t) | } \sum_{ i \in \mathcal{X}^{C} (t) } \| \nabla f_{i} ( \bm{w} (t) ) \| \right]^{2} \right)  \nonumber \\
			& \leq 4 \, \mathbb{E} \left( \left( \dfrac{D - | \mathcal{X} (t) |}{D} \right)^{2} ( \xi_{1} + \xi_{2} \| \nabla F ( \bm{w} (t) ) \|^{2} ) \right),  \nonumber
		\end{align}
	where the triangle inequality is applied twice, and the resulting terms are simplified. Substituting this result into \eqref{SA-EQA2b} gives:
		\begin{equation} \nonumber
			\begin{aligned}
				\lefteqn{\mathbb{E} \left( F ( \bm{w} (t + 1) ) \right)} \nonumber \\
				& \leq \mathbb{E} \left( F ( \bm{w} (t) ) \right) - \dfrac{1}{ 2 L } \| \nabla F ( \bm{w} (t) ) \|^{2} + \dfrac{1}{ 2 L } \mathbb{E} ( \| \bm{o} \|^{2} ) \\
				& \leq \mathbb{E} \left( F ( \bm{w} (t) ) - \dfrac{1}{ 2 L } \| \nabla F ( \bm{w} (t) ) \|^{2} + \dfrac{ 2 \alpha }{ L } ( \xi_{1} + \xi_{2} \| \nabla F ( \bm{w} (t) ) \|^{2} ) \right) \\
				& \leq \mathbb{E} \left( F ( \bm{w} (t) ) + \left( \dfrac{ 2 \xi_{2} A }{L} - \dfrac{1}{ 2 L } \right) \| \nabla F ( \bm{w} (t) ) \|^{2} + \dfrac{ 2 \xi_{1} A }{L} \right) \\
				& \leq \mathbb{E} \left( F ( \bm{w} (t) ) + \left( 4 \xi_{2} \alpha^{2} - 1 \right) ( F ( \bm{w} (t) ) - F ( \bm{w}^{*} ) ) + \dfrac{ 2 \xi_{1} A }{L} \right),
			\end{aligned}
		\end{equation}
	where $\alpha \triangleq \left( {D - | \mathcal{X} (t) |} / {D} \right)^{2}$. Therefore, we obtain
		\begin{equation} \nonumber
			\begin{aligned}
				\lefteqn{\mathbb{E} \left( F ( \bm{w} (t + 1) ) - F ( \bm{w}^{*} ) \right)} \nonumber \\
				& \leq \mathbb{E} \left( F ( \bm{w} (t) ) - F ( \bm{w}^{*} ) + \left( 4 \xi_{2} \alpha^{2} - 1 \right) ( F ( \bm{w} (t) ) - F ( \bm{w}^{*} ) ) \right.  \nonumber \\
				& \quad \left. {}+ \dfrac{ 2 \alpha \xi_{1} }{L} \right) \\
				& \leq \mathbb{E} \left( 4 \xi_{2} \alpha^{2} ( F ( \bm{w} (t) ) - F ( \bm{w}^{*} ) ) + \dfrac{ 2 \alpha \xi_{1} }{L} \right) .
			\end{aligned}
		\end{equation}
	Additionally, we can further derive:
		\begin{equation} \nonumber
			\begin{aligned}
				& \mathbb{E} \left( F ( \bm{w} (t + 1) ) - F ( \bm{w}^{*} ) - \dfrac{2 \alpha \xi_{1}}{ ( 1 - 4 \xi_{2} A ) L } \right) \\
				& \leq \mathbb{E} \left( 4 \xi_{2} A \left( F ( \bm{w} (t) ) - F ( \bm{w}^{*} ) - \dfrac{2 \alpha \xi_{1}}{ ( 1 - 4 \xi_{2} A ) L } \right) \right) \\
				& \leq \mathbb{E} \left( ( 4 \xi_{2} A )^{t + 1} \left( F ( \bm{w} (0) ) - F ( \bm{w}^{*} ) - \dfrac{2 \alpha \xi_{1}}{ ( 1 - 4 \xi_{2} A ) L } \right) \right),
			\end{aligned}
		\end{equation}
	which concludes \eqref{S3-EQ7} in Theorem~\ref{S3-TM1}. This completes the proof.

\section{Proof of Proposition~\ref{S4-PR1}} \label{App-Prop1}

Recalling the objective function of $\mathcal{P}_{3}$, define the (unsquared) learning-progress term
\begin{equation}\label{S4-EQpsi-def}
\psi(\bm p)\triangleq
\dfrac{BT}{V\ln 2}\sum_{k=1}^{K}
\ln\!\Bigg(1+\dfrac{G_{k,k}p_k}{\sum_{\ell\neq k}G_{k,\ell}p_\ell+\sigma^2}\Bigg)
+A-D,
\end{equation}
so that $\Phi(\bm p)=\big(\psi(\bm p)\big)^2$. Likewise, let $\widetilde{\psi}(\bm p\mid \bm p^*)$ denote the expression inside the square in \eqref{S4-EQ12}, so that
$\widetilde{\Phi}(\bm p\mid \bm p^*)=\big(\widetilde{\psi}(\bm p\mid \bm p^*)\big)^2$.

\emph{Step 1: A lower bound on each log-SINR term.}
For each $k$, rewrite the log-SINR term as a difference of logarithms:
\begin{align}\label{S4-EQlog-diff}
\lefteqn{\ln\!\Bigg(1+\dfrac{G_{k,k}p_k}{\sum_{\ell\neq k}G_{k,\ell}p_\ell+\sigma^2}\Bigg)} \nonumber \\
& =
\ln\!\Bigg(1+\sum_{\ell=1}^{K}\dfrac{G_{k,\ell}p_\ell}{\sigma^2}\Bigg)
-
\ln\!\Bigg(1+\sum_{\ell\neq k}\dfrac{G_{k,\ell}p_\ell}{\sigma^2}\Bigg).
\end{align}
Let
\[
u_k(\bm p)\triangleq 1+\sum_{\ell\neq k}\dfrac{G_{k,\ell}p_\ell}{\sigma^2},\qquad
u_k(\bm p^*)\triangleq 1+\sum_{\ell\neq k}\dfrac{G_{k,\ell}p_\ell^*}{\sigma^2}.
\]
Since $\ln(x)$ is concave on $x>0$, its first-order Taylor expansion at $x'=u_k(\bm p^*)$ yields the global \emph{upper} bound
\begin{equation}\label{S4-EQlog-UB}
\ln\big(u_k(\bm p)\big)
\le
\ln\big(u_k(\bm p^*)\big)
+\frac{1}{u_k(\bm p^*)}\big(u_k(\bm p)-u_k(\bm p^*)\big),
\end{equation}
with equality at $\bm p=\bm p^*$. Multiplying \eqref{S4-EQlog-UB} by $-1$ gives the corresponding global \emph{lower} bound on $-\ln(\cdot)$:
\begin{equation}\label{S4-EQminuslog-LB}
-\ln\big(u_k(\bm p)\big)
\ge
-\ln\big(u_k(\bm p^*)\big)
-\frac{1}{u_k(\bm p^*)}\big(u_k(\bm p)-u_k(\bm p^*)\big),
\end{equation}
which is tight at $\bm p=\bm p^*$. Substituting \eqref{S4-EQminuslog-LB} into \eqref{S4-EQlog-diff} yields, for each $k$,
\begin{align}\label{S4-EQterm-LB}
\lefteqn{\ln\!\Bigg(1+\dfrac{G_{k,k}p_k}{\sum_{\ell\neq k}G_{k,\ell}p_\ell+\sigma^2}\Bigg)} \nonumber \\
& \ge
\ln\!\Bigg(1+\sum_{\ell=1}^{K}\dfrac{G_{k,\ell}p_\ell}{\sigma^2}\Bigg)
-\ln\big(u_k(\bm p^*)\big)
+1-\frac{u_k(\bm p)}{u_k(\bm p^*)},
\end{align}
which is exactly the inner surrogate term used in \eqref{S4-EQ12}.

\emph{Step 2: Lower bound on $\psi(\bm p)$ and majorization of $\Phi(\bm p)$.}
Summing \eqref{S4-EQterm-LB} over $k$ and scaling by $\frac{BT}{V\ln 2}$ yields
\begin{equation}\label{S4-EQpsi-LB}
\psi(\bm p)\ \ge\ \widetilde{\psi}(\bm p\mid \bm p^*),\quad \forall \bm p.
\end{equation}
In the considered data-deficit operating regime, $\psi(\bm p)\ge 0$ holds over the feasible set, and the surrogate construction also ensures $\widetilde{\psi}(\bm p\mid \bm p^*)\ge 0$. Therefore, the pointwise inequality \eqref{S4-EQpsi-LB} implies
\[
\Phi(\bm p)=\psi(\bm p)^2 \ \le\ \widetilde{\psi}(\bm p\mid \bm p^*)^2=\widetilde{\Phi}(\bm p\mid \bm p^*),
\]
which establishes the majorization property in Proposition~\ref{S4-PR1}.

\emph{Step 3: Tangency and first-order consistency.}
The bound \eqref{S4-EQlog-UB}--\eqref{S4-EQminuslog-LB} is tight at $\bm p=\bm p^*$ because it is the first-order Taylor expansion at $u_k(\bm p^*)$. Hence \eqref{S4-EQterm-LB} holds with equality at $\bm p=\bm p^*$ for every $k$, which implies
$\widetilde{\psi}(\bm p^*\mid \bm p^*)=\psi(\bm p^*)$ and thus
$\widetilde{\Phi}(\bm p^*\mid \bm p^*)=\Phi(\bm p^*)$.
Moreover, since the surrogate is constructed via first-order Taylor expansion at $\bm p^*$, the gradients coincide at the expansion point, i.e.,
$\nabla_{\bm p}\widetilde{\Phi}(\bm p\mid \bm p^*)\big|_{\bm p=\bm p^*}=\nabla_{\bm p}\Phi(\bm p^*)$.
This completes the proof.

\section{Proof of Proposition~\ref{S5-PR3}} \label{App-Prop2}
Fix $\{\bm p_i\}_{i=1}^I$ and consider minimizing the ALF \eqref{S4-EQ14} w.r.t. $\{P_i\}_{i=1}^I$ under the coupling constraint $\sum_{i=1}^I P_i = P$ (and $P_i \ge 0$). To obtain a distributed-friendly form, we enforce the consensus structure $\beta_i=\beta$ for all $i$ (this is standard in dual decomposition / augmented-Lagrangian treatments).

Introduce a Lagrange multiplier $\nu$ for the equality $\sum_i P_i=P$. The partial derivatives of the augmented Lagrangian w.r.t. $P_i$ give the optimality conditions
\begin{equation}
	\frac{\partial}{\partial P_i}\Big[ \sum_{j=1}^I \beta( \mathbf{1}^T\mathbf{p}_j - P_j)+\frac{\mu}{2}\sum_{j=1}^I(\mathbf{1}^T\mathbf{p}_j-P_j)^2 \Big] - \nu =0,
\end{equation}
which reduces to
\begin{equation}
	-\,\beta -\mu(\mathbf{1}^T\mathbf{p}_i-P_i) -\nu =0
		\, \Longrightarrow \,
	P_i = \mathbf{1}^T\mathbf{p}_i - \frac{\beta+\nu}{\mu}.
\end{equation}
Summing this identity over $i=1, \cdots, I$ and using $\sum_i P_i=P$ yields
\begin{equation} 
	P = \sum_{i=1}^I \mathbf{1}^T\mathbf{p}_i - I\frac{\beta+\nu}{\mu}
	\, \Longrightarrow \,
\frac{\beta+\nu}{\mu} = \frac{\sum_{i}\mathbf{1}^T\mathbf{p}_i - P}{I}.
\end{equation}
Substituting back gives the intended \eqref{S5-EQ15}. Finally, treating $\beta$ as the dual variable associated with the (global) equality constraint, a projected gradient (dual ascent) step for $\beta$ on the augmented Lagrangian with step-size $\mu/I$ yields the desired \eqref{S5-EQ16}. This completes the proof.

\bibliographystyle{IEEEtran}
\bibliography{MyRef.bib}

\begin{IEEEbiography}
	[{\includegraphics[width=1in, height=1.25in, clip, keepaspectratio]{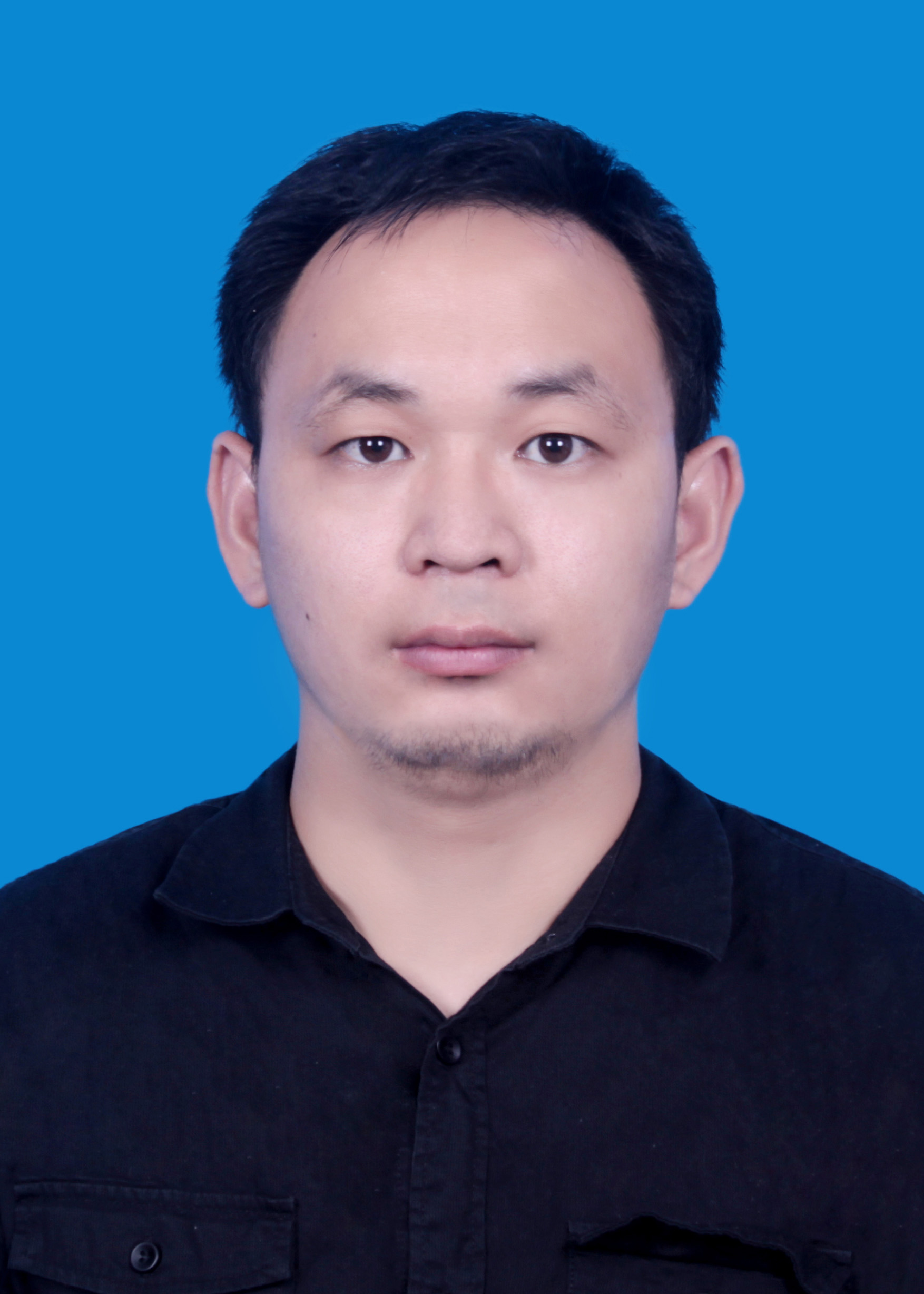}}]{Haihui Xie} received the B.S. degree and the M.S. degree in photonic and electronic engineering from Fujian Normal University, Fuzhou, China, in 2014 and 2016, respectively, and the Ph.D. degree in information and communication engineering from Sun Yat-sen University, Guangzhou, China, in 2023. He is currently a Lecturer at the School of Fujian Agriculture and Forestry University (FAFU), Fuzhou. His research interests include edge learning, optimization, and wireless communications.
\end{IEEEbiography}

\begin{IEEEbiography}
	[{\includegraphics[width=1in, height=1.25in, clip, keepaspectratio]{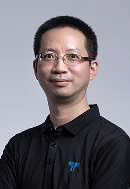}}]{Wenkun Wen} (Member, IEEE) received the Ph.D. degree in Telecommunications and Information Systems from Sun Yat-sen University, Guangzhou, China, in 2007. Since 2020, he has been with Techphant Technologies Co. Ltd., Guangzhou, China, as Chief Engineer.

From 2008 to 2009, he was with the Guangdong-Nortel R\&D center in Guangzhou, China, where he worked as a system engineer for 4G systems. From 2009 to 2012, he worked at the LTE R\&D center of New Postcom Equipment Co. Ltd., Guangzhou, China, where he served as the 4G standard team manager. From 2012 to 2018, he was with the 7th Institute of China Electronic Technology Corporation (CETC) as an expert in wireless communications. From 2018 to 2020, he served as Deputy Director of the 5G Innovation Center at CETC. His research interests include 5G/B5G mobile communications, machine-type communications, narrow-band wireless communications, and signal processing.
\end{IEEEbiography}	

\vspace{-10pt}

\begin{IEEEbiography}
	[{\includegraphics[width=1in, height=1.25in, clip, keepaspectratio]{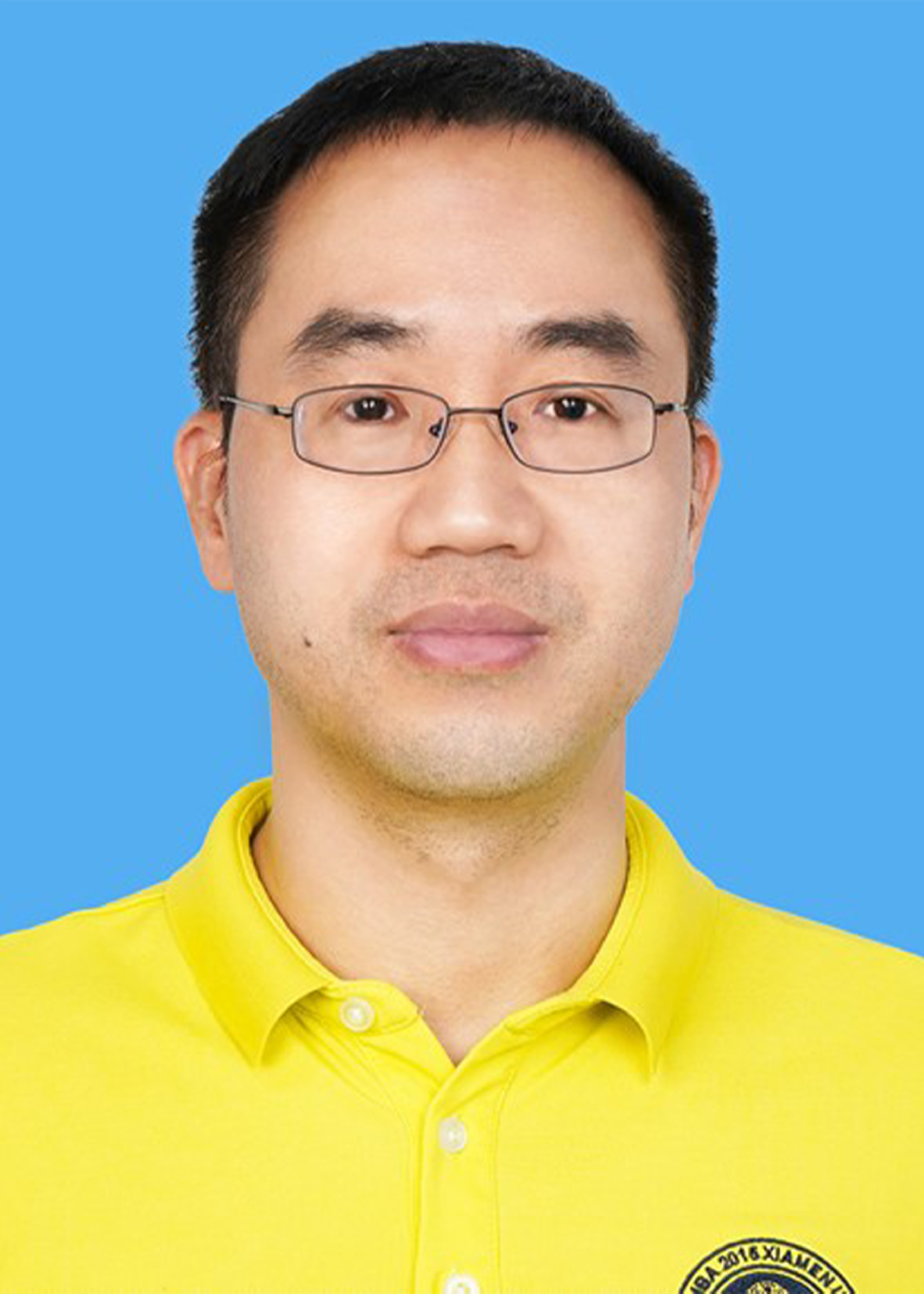}}]{Shuwu Chen} received his M.S. degree in Radio Physics from Xiamen University in 2003. He is currently a Professor at the College of Computer and Information Sciences, Fujian Agriculture and Forestry University. He also serves as Director of the Engineering Research Center of Smart Sensing and Agricultural Chip Technology, Fujian Province University. His research interests include the Internet of Things, edge computing, and artificial intelligence algorithms.
\end{IEEEbiography}

\vspace{-10pt}

\begin{IEEEbiography}
	[{\includegraphics[width=1in, height=1.25in, clip, keepaspectratio]{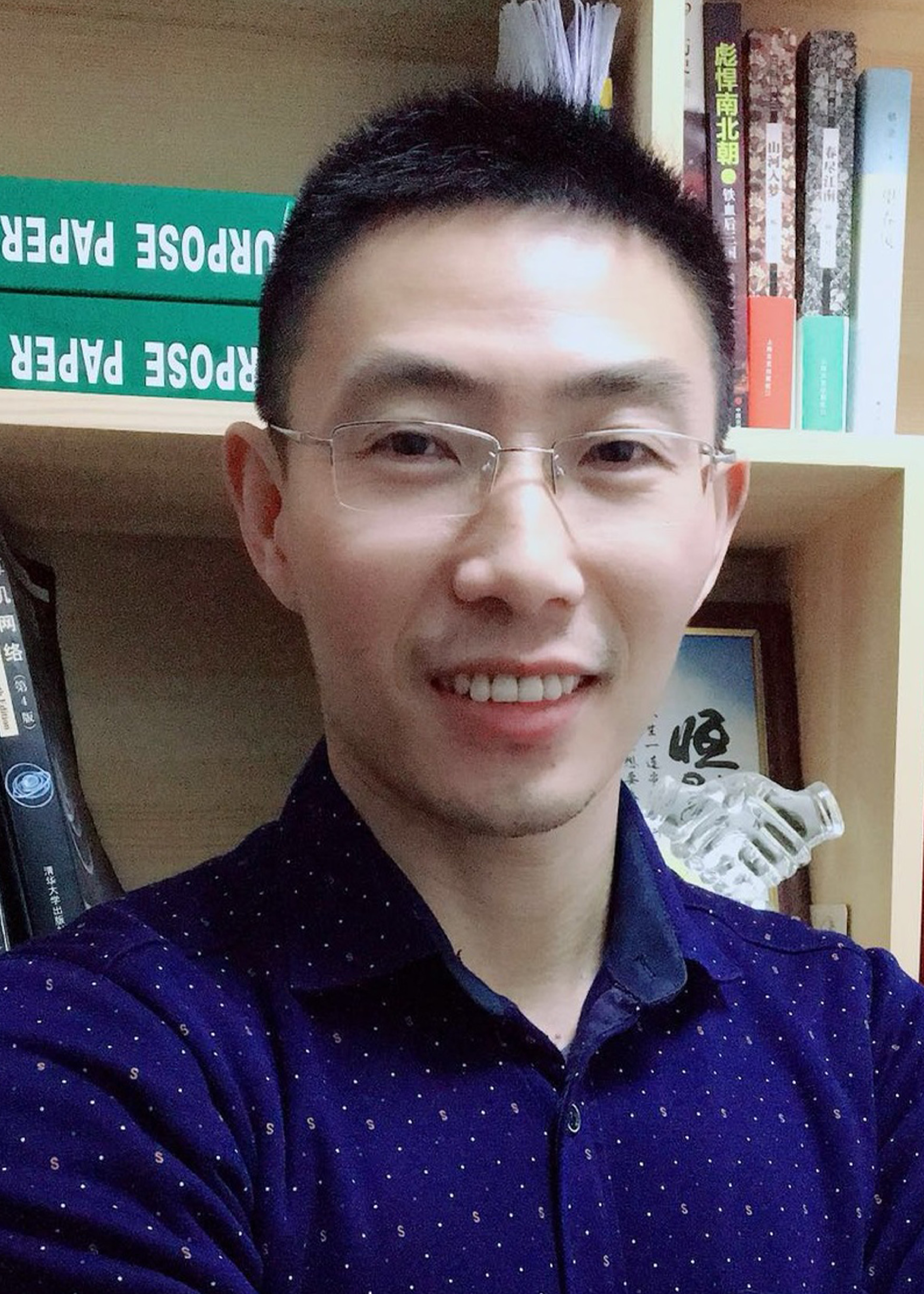}}]{Zhaogang Shu} (Senior Member, IEEE) received his B.S. and M.S. degrees in computer science from Shantou University, China, in 2002 and 2005, respectively,  a Ph.D. degree from South China University of Technology, Guangzhou, China, in 2008. 
	
	He is currently an Associate Professor in the College of Computer and Information Science at Fujian Agriculture and Forestry University, Fuzhou, China, and the director of the Network System and Cloud Computing Lab at the university. He is a senior member of the China Computer Federation and the Fujian Computer Society. From Sept. 2008 to July 2012, he served as a senior engineer and project manager at Ruijie Network Corporation in Fuzhou, China. From Oct. 2018 to Oct. 2019, he worked as a visiting professor in the MOSIAC lab at the Department of Communications and Networking, School of Electrical Engineering, Aalto University, Finland. He led more than 10 research projects and authored more than 40 papers and 10 patents. His research interests include software-defined networks, network function virtualization, 5G networks, next-generation network architecture, network security, machine learning-based network optimization, cloud computing, and edge computing. 
\end{IEEEbiography}


\vspace{-10pt}

\begin{IEEEbiography}
	[{\includegraphics[width=1in, height=1.25in, clip, keepaspectratio]{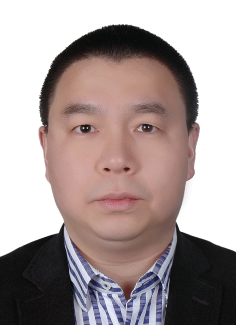}}]{Minghua Xia} (Senior Member, IEEE) received the Ph.D. degree in Telecommunications and Information Systems from Sun Yat-sen University, Guangzhou, China, in 2007.
	
	From 2007 to 2009, he was with the Electronics and Telecommunications Research Institute (ETRI) of South Korea, Beijing R\&D Center, Beijing, China, where he worked as a member and then as a senior member of the engineering staff. From 2010 to 2014, he was in sequence with The University of Hong Kong, Hong Kong, China; King Abdullah University of Science and Technology, Jeddah, Saudi Arabia; and the Institut National de la Recherche Scientifique (INRS), University of Quebec, Montreal, Canada, as a Postdoctoral Fellow. Since 2015, he has been a Professor at Sun Yat-sen University. Since 2019, he has also served as an Adjunct Professor at the Southern Marine Science and Engineering Guangdong Laboratory (Zhuhai). His research interests are in the general areas of wireless communications and signal processing.
\end{IEEEbiography}

\vfill
\end{document}